\documentclass[letterpaper]{article}
\usepackage{uai2018}
\usepackage[margin=1in]{geometry}
\usepackage{times}

\usepackage{amsmath, amssymb, amsthm}
\usepackage{natbib}
\usepackage{color}
\usepackage{algorithm}
\usepackage{algpseudocode}
\usepackage{todonotes}
\DeclareMathOperator{\KL}{KL}
\DeclareMathOperator{\kl}{kl}
\DeclareMathOperator{\EE}{\mathbb{E}}
\DeclareMathOperator{\PP}{\mathbb{P}}
\DeclareMathOperator{\R}{\mathbb{R}}
\DeclareMathOperator{\N}{\mathbb{N}}
\DeclareMathOperator{\Ber}{Ber}

\newtheorem{lemma}{Lemma}
\newtheorem{theorem}{Theorem}

\theoremstyle{definition}

\newtheorem{definition}{Definition}

\title{Bandits with Side Observations: Bounded vs. Logarithmic Regret}

\author{ {\bf R\'emy Degenne
} \\
LPSM, CNRS, Sorbonne Universit\'e, \\
Universit\'e Paris Diderot, 75013 Paris, France; \\
CMLA, ENS Cachan, CNRS, \\
Universit\'e Paris-Saclay,
94235 Cachan, France\\
\And
{\bf Evrard Garcelon} \\
CMLA, ENS Cachan,\\
94235 Cachan, France \\
\And
{\bf Vianney Perchet}  \\
CMLA, ENS Cachan, CNRS, \\
Universit\'e Paris-Saclay, 
94235 Cachan, France\\
Criteo AI Lab\\
75009 Paris
}

\begin{document}

\maketitle

\begin{abstract}
We consider the classical stochastic multi-armed bandit but where, from time to time and roughly with frequency $\epsilon$, an extra observation is gathered by the agent for free. We prove that, no matter how small $\epsilon$ is the agent can ensure a regret uniformly bounded in time.

More precisely, we construct an algorithm with a regret smaller than $\sum_i \frac{\log(1/\epsilon)}{\Delta_i}$, up to multiplicative constant and $\log\log$ terms. We also prove a matching lower-bound, stating that no reasonable algorithm can outperform this quantity.
\end{abstract}

\section{INTRODUCTION}

We consider the celebrated multi-armed bandit framework (sometimes also called online learning), a repeated decision problem where an agent (or an algorithm, a machine, a player, etc.) takes sequentially decisions from a finite set. Each decision  gives a stochastic reward to the  agent of fixed expectation. The main objective is to derive an algorithm maximizing the cumulative reward or minimizing its normalized version, the so-called ``regret''. The latter is simply the difference between the cumulative expected reward of an agent knowing in hindsight the optimal decision, and the cumulative reward of the algorithm.

Online learning can be traced back to the 30's, when Thompson analysed random clinical trial using an analogy with finding the best slot-machine in a casino by pulling sequentially their arms in order to minimize the total loss. During the 20th century, many improvements have been made, at least on the asymptotic version of the problem.  The quantity of  theoretical studies and practical applications of bandits have exploded since the early 2000. There are several reasons for that. First of all, a simple yet almost optimal algorithm called UCB has been developed. Its simple structure allows to adapt it to many different settings. As a consequence, many possible applications of online learning have been developed. Amongst them, we can mention the routing problem: given a network with congested edges, one must find the quickest way from some origin to a destination (this setting incorporates a combinatorial structure); this can be used to send packets in a network, as well as finding the quickest itinerary from a point A to a point B.  Online advertising is another application: given a possible set of ads, one must find the ad with the highest probability of click. The last application we mention is concerned with wireless network and/or cognitive radio, where either a radio can change from an available channel to other channels to improve its reception or emission quality, or alternatively a wireless source, in a relay selection problem where multiple relays are available, can explore those nodes to achieve better transmissions rates. 
One of the typical and crucial assumption of all these models is that the agent only observes the outcome of his decisions, but not what the other decisions would have given him. For instance, using a slot machine only gives you a feedback on the performance of that very machine, displaying an ad only gives information of the probability of clicks on that specific ad, etc. This assumption is actually called ``bandit feedback''. At the other end of the spectrum, the dual assumption (mostly used in non-stationary environment that we are not concerned with in that paper) is the ``full information feedback'', where all the outcomes of all decisions are observed at all stages. However, none of our motivating examples satisfies this strong assumption.

However, we argue that the bandit feedback is also too strong and that in many cases more informations are available to the agent. Typically, the agent will always observe the outcome of his own decision, but with some small probability he might also get one (or several, but that is irrelevant to our setting) extra ``free'' information. For instance, consider the original multi-armed bandit problem. A gambler is in a casino and wants to find out which slot machine is the best one. From time to time, he might observe other gamblers playing nearby machines. Even if this does not cost him anything, he gets feedback on the other machines. This effect also appears in other settings.  In wireless network, a source with an allocated transmission capacity (because of a power-saving allocation protocol for instance)  sends data through a relay and  may have the opportunity to send another custom packet (so that the energy needed to send this packet is less than the available energy) through another relay in order to estimate transmissions rates. In online advertisement (and actually many other industrial markets), companies are willing to spend a small fraction of their data, say with probability $\varepsilon$ as in the celebrated $\varepsilon$-greedy algorithm, just to acquire new information. An algorithm is only evaluated on the remaining (of proportion $1-\varepsilon$) fraction of the data treated. In a multi-armed bandit setting, this means that with probability $\varepsilon$, the next decision is ``free''. Finally, we can also think that in the congested network problem, an algorithm can from time to time send ``fake'', but free, packets to test the congestion; conversely, an app trying to minimize the congestion time of its users might be able to use free information if it notices that a bucket of users (for instance, those that are registered) might explore new road willingly, i.e., without uninstalling the app.

We therefore focus on the classical multi-armed bandits but where some extra and free information is available from time to time. Clearly, if the probability $\varepsilon$ that it happens is arbitrarily close to 0, the improvement will be negligible. But we aim at constructing ``optimal'' algorithm, i.e., whose regret is small and in a multiplicative constant of the best regret achievable regret by ``meaningful'' algorithms. All these concepts are explained in details in the remaining of the paper that is organized as follows.

The model is introduced in Section \ref{SE:Model}, where we provide a very na\"ive algorithm achieving bounded regret (uniformly in time).  We exhibit in Section \ref{SE:Lower} non-trivial lower bounds (we emphasize here that traditional bandit lower bounds are void in our setting). Algorithms are described and analysed in Section \ref{SE:Upper}. Finally, Section \ref{SE:Expe} is dedicated to experiments illustrating the different guarantees and dependencies in the parameters of the models.

\subsection{RELATED WORKS}

This paper is not the first one to consider additional, free informations, available to the agents while optimizing. There are many different ways of modelling this idea, but our paper is the first one (to our knowledge) that also focus on strategical aspects of obtaining these free informations to reduce regret, especially in the stochastic case. 

There exists models where when a specific decision is taken, automatically (resp.\ with some probability), the performance of some other decision are observed \citep{alon2015online, chen2016combinatorial, caron2012leveraging}. Those models assume that there exists a directed (resp.\ weighted) graph whose  set of nodes is the set of  decisions. When the agent takes a decision, he also observes the outcome of any node  linked  (resp.\ with a probability proportional to the weight of the edge) to the current decision node. Our passive model could be recast as a specific case of that setting, but our results are much finer than the ones available for the general case.

In \citep{yu2009piecewise} the rewards are stochastic but their means change at unknown time points. Free additional informations are queried by the algorithm in order to detect these change points. They however are not used to decrease the regret of the base bandit algorithm. 

Another trend of literature of additional free information in multi-armed bandit studies the ``adversarial'' case, where no stationary assumption is made on the sequence of rewards (namely, there are not i.i.d.)\citep{audibert2010regret, cesa2006regret,mannor2011bandits}. However the rate of convergence in the two extreme cases (bandit and full information) have the same dependency in $T$, the total number of stages. To be precise, the regret is either of the order of $\sqrt{KT}$ (in the bandit case) or $\sqrt{\log(K)T}$ (in the full information case), where $K$ is the number of decisions. Intermediate settings (where $1+M$ observations are available at each stage) interpolate between those two cases.

In the stochastic case though, regret is uniformly bounded with full information and grows logarithmically in the bandit case. As a consequence, even the rate of convergence will depend on the size of free informations.

\section{MULTI-ARMED BANDITS, REGRET MINIMIZATION AND FEEDBACKS} \label{SE:Model}

In that section, we describe precisely the stochastic multi-armed bandit problem and its objective, the minimization of regret.

\subsection{STOCHASTIC MULTI-ARMED BANDITS}

\subsubsection{Bandit vs Full-Information}

At each successive stage $t\in\N^*$, an agent  takes a decision (or \textit{pulls an  arm} using the multi-armed bandit lingo) $i_t$ in the finite set $[K]:=\{1,\ldots,K\}$. After pulling this arm, the agent receives the reward $X_t^{(i_t)}\in \R$, which is sampled from a real reward distribution $\nu^{(i_t)}$ of expectation $\mu^{(i_t)}$. As a consequence,  the stochastic bandit problem is parametrised  by the vector of distribution, $(\nu^{(1)},\ldots,\nu^{(K)})$, or alternatively in the non-parametric case, by the vector of expected rewards  $(\mu^{(1)},\ldots,\mu^{(K)})$. Throughout the paper, the results are stated using the arbitrary ordering $\mu^{(1)}>\mu^{(2)}\geq \ldots\geq\mu^{(K)}$. Obviously, those vectors are unknown to the agent, who is aiming at optimizing her cumulative expected reward $\sum_{t=1}^T \mu^{(i_t)}$. Actually, instead of this cumulative reward, the objective is normalized into  \textit{cumulative regret} minimization.

The cumulative regret (or simply regret) of an algorithm at stage $T$ is defined as
\begin{align*}
R_T = T\max_{i\in[K]}\mu^{(i)} - \sum_{t=1}^T\mu^{(i_t)} \: ,
\end{align*}
i.e., it is the difference between the maximal possible cumulative reward up to stage $T$ and the expectation of the reward gained by the successive choices of arms $i_1,\ldots,i_T$. Following the classical notations, we define $\mu^\star = \max_{i \in [K]} \mu^{(i)}$ and the gaps $\Delta_i = \mu^\star-\mu^{(i)}$. In the non-parametric case, these gaps are the relevant quantities characterising the complexity of a bandit problem.

There are different standard assumption on the feedbacks available to the agent before taking a new decision.  In the \textit{bandit} setting, she observes only her reward $X_t^{(i_t)}$ (and, specifically, not the other $X_t^{(k)}$) at the end of stage $t \in \N^*$. In the \textit{full information} setting, she observes the full vector of rewards $(X_t^{(1)},\ldots,X_t^{(K)})\in\R^K$. With full information, the \textit{Follow The Leader} (FTL) algorithm that selects the arg max of the empirical average $\overline{X}_t^{(i)}:=\frac{1}{t}\sum_{s=1}^t X_s^{(i)}$ attains a uniformly bounded  regret (with respect to $T$). In the bandit setting, FTL gets a linear regret, yet the logarithmic optimal dependency in $T$ is achieved by many algorithms. One of the most popular, called \textit{Upper Confidence Bound} (UCB), selects the argmax of the empirical average augmented of an error term $\hat{\mu}_t^{(i)} + \sqrt{6\frac{\log(t)}{N_i(t)}}$ where $N_i(t)$ is the number of pulls of arm $i$ up to stage $t$, while $\hat{\mu}_t^{(i)}:=\frac{1}{N_i(t)}\sum_{s: i_s =i} X_s^{(i_s)}$.

Many other algorithms are variants of UCB, by modifying the error term, changing some parameters, specifying it for a given class of parametric distributions, etc.
\subsubsection{Additional Informations}

As specified and motivated in the Introduction, we aim at analysing  intermediate settings between bandit and full information, in which a subset of the reward vector might be observed. More precisely, at some stages, the agent not only observes an arm by pulling it but might also observe a second arm for free, i.e., without getting a reward (and without incurring any regret).  We consider several ways in which these free observations can be obtained: they can be deterministically available periodically (for instance every $1/\varepsilon$ rounds) or arrive randomly (at each stage with probability $\varepsilon$); the agent can also be a \textit{passive observer} if she can not choose from which arm she gets an extra information (the environment chooses it for her, in a manner to be specified latter on), or she can be an \textit{active observer} if she can choose the arm to observe freely.

We end this section with some notations. In the random time arrival of free information, we assume that at each stage $t\in \N^*$ a Bernoulli random variable $Z_t$ with expectation $\epsilon_t$ (whose law is denoted by $\Ber(\epsilon_t)$) is sampled and a free observation is available if $Z_t = 1$. The particular setting in which $\epsilon_t$ is constant will be called static random. We will denote by $i_t$ the arm pulled and by $f_t$ the arm chosen to be observed using the free information (if available). The total number of pulls of arm $i$ up to stage $t$ is $N_i(t)$, the number of free observations $F_i(t)$ and the total number of observation of arm $i$ is $O_i(t) = N_i(t) + F_i(t)$.

\subsection{A FINITE REGRET SETTING}

It is not really difficult to devise a na\"ive algorithm with a (uniformly) bounded regret at least in the deterministic case, when a free observation is obtained every $1/\epsilon$ round. We consider for simplicity the case of $K=2$ arms in this section as it gives all the intuitions. Consider the following (heavily sub-optimal) strategy, which we denote by FTL-robin: pull the \textit{leading arm} (the one with the highest empirical average $\hat{\mu}_t^{(i)}$) and when a free sample is available, observe arms in a round-robin fashion.

After a period of $1/\epsilon$ stages, both arms have their observation counters increased by at least one. As a consequence,  this simple algorithm FTL-robin can be seen as a full-information algorithm which would take $1/\epsilon$ stages to get the observations. To simplify intuitions
\begin{lemma}\label{lemma:FTL-robin}
The regret of the FTL-robin algorithm on the deterministic setting with $K=2$ satisfies
\begin{align*}
\EE R_T \leq \frac{\overline{c}}{\epsilon}\frac{1}{\Delta} \: , \ \text{ where } \  \Delta = |\mu^{(1)}-\mu^{(2)}|,
\end{align*}
and there exist distributions $(\nu^{(1)},\nu^{(2)})$ such that 
\begin{align*}
 \frac{\underline{c}}{\epsilon}\frac{1}{\Delta} \leq \EE R_T,
\end{align*}
where $\underline{c}, \overline{c} >0$ are universal constants that do not involve any parameter of the problem.
\end{lemma}
This lemma shows that even the simplest algorithm gets a finite regret in this setting. The proof is almost trivial and omitted. To provide some insights, just assume that $\nu^{(1)} = \mathcal{N}(\Delta,1)$ and $\nu^{(2)} = \delta_0$. Then the regret of FTL-robin is equal to the $\Delta/\epsilon$ times the number of times that $\overline{X}_t^{(1)}$ is smaller than 0. Basic computations show that this number is of order $\frac{1}{\Delta^2}$.

The relevant question is then not the asymptotic regime, but what is the precise optimal dependency on $\epsilon$. Indeed, when $\epsilon < \frac{1}{\log T}$, this bound gets larger than the $O(\log T)$ regret of another naive approach, which is to use an algorithm for bandits and discard the additional information.

This free information problem is characterized by a transition from "small" $\epsilon$, where the amount of additional information is not enough to improve the performance of bandit algorithm, to "big" $\epsilon$, where the regret is finite and the setting is closer to full-information.

We answer the question of what "small" and "big" mean in this context and where the transition occurs and we display algorithms enjoying both logarithmic regret when $\epsilon$ is small and finite regret when it is big.

\section{LOWER BOUNDS}\label{SE:Lower}

We first consider the definition of \textit{optimality} of an algorithm, that is, what is the minimal regret achievable by any "reasonable" algorithm, in a sense we will make precise. Our lower bounds will highlight a transition from logarithmic (with respect to the horizon $T$) to finite regimes when $\epsilon$ gets big enough.

There are now quite standard techniques to devise lower bounds for stochastic bandits problems, but surprisingly these techniques are inadequate in our case, due to the finiteness of the optimal regret. As a finite regret is possible, a traditional, asymptotic lower bound for $\frac{\EE R_T}{\log T}$ \citep{lai1985asymptotically} could only be 0 and hence would not be informative. We can obtain a finite time version of this type of bound as in \citep{garivier2016explore} by imposing that our algorithm should perform better than a reference algorithm.
\begin{definition}\label{def:better_than_ucb}
An algorithm is said to be sub-logarithmic with constants $C$, $C_0$ if on all bandit problems it verifies for all stages $T\in\N^*$,
\begin{align*}
\EE R_T \leq C\sum_{i=1}^K\frac{\log T}{\Delta_i} + C_0 \sum_{i=2}^K \Delta_i \: .
\end{align*}
\end{definition}
There exists sub-logarithmic algorithms (UCB for example, with constants $C=8$, $C_0=(1+\pi^2/3)$ \citep{auer2002finite}). A sub-logarithmic algorithm is performing at least as good as the UCB baseline. This finite time constraint on the performance of the algorithm translates into a lower bound: to perform relatively well on all bandit problems, an algorithm cannot outperform the lower bound guarantee on any of them.

\subsection{PASSIVE OBSERVER}\label{SE:LowerPassive}

When the observer is passive (i.e., she does not choose the arm $f_t$ to observe freely), we assume that $f_t$ is equal to $i\in[K]$ with probability $p_t^{(i)}$ chosen by the environment. Consider the static setting in which for all $t$, $Z_t \sim \Ber(\epsilon)$ and the probabilities $p_t^{(i)}$ do not depend on the stage $t$ (we will thereafter omit the subscript $t$).

Standard lower bound techniques proceed as follows: at stage $T$, the expected number of pulls of an arm is linked to the Kullback-Leibler divergence between the bandit problem studied and a related alternative, in which this arm would be the best one (roughly speaking, in order to be able to ``test'' that the problem is not the alternative one, a minimum number of samples of that arm must be gathered in the original problem). 

A bound on this divergence gives a constraint of the form $\EE O_i(T) \geq h_i(t)/\Delta_i^2$ for some function $h_i(T) = O(\log T)$. Then a lower bound for the regret is the minimal value of $\sum_{i=2}^K\Delta_i \EE N_i(t)$ respecting all these constraints, that can be computed through some linear program. With this proof technique, we obtain lemma \ref{lemma:LB_passive_simple} .
\begin{lemma}\label{lemma:LB_passive_simple}
The regret of a sub-logarithmic algorithm with constants $C$, $C_0$ must verify
\begin{align*}
\EE_1R_T
&\geq \sum_{i=2}^K \max\left\{0, \frac{h_i(T)}{2\Delta_i} - \epsilon p^{(i)} T \Delta_i) \right\}\: .
\end{align*}
where $h_i(T) = O(\log T)$ (see appendix for a detailed definition).
\end{lemma}
As mentioned above, this lower bound  is void as it reaches  0 as soon as $T$ is big enough, bigger than $ \frac{1}{\epsilon}\max_{j\geq2} \frac{ h_j(T)}{2p^{(j)}\Delta_j^2}$.

We want to explain why this lower bound fails to provide relevant informations as our algorithm (see Section \ref{SE:Upper}) are somehow inspired by this. Recall that the lower bound only states that any reasonable algorithm must have gathered, for each sub-optimal arm, a given number of observations, namely $\frac{h_i(T)}{2\Delta_i^2}$. However, $h_i(T)$ grows sub-linearly, while the number of free observations grows linearly. So if $T$ is large enough, there will be in total enough free observations to allocate $\frac{h_i(T)}{2\Delta_i^2}$  of them to arm $i$ and an optimal algorithm should somehow have used only free information to explore. 

However, this is only possible if the $\varepsilon T$ free observations were gathered  \textit{at the beginning} of the problem and not scarcely with time! Indeed, in the traditional lower bounds techniques, the fact that arm $i$ is observed at the beginning or at the end of time is irrelevant (since the cost of one pull is constant throughout time).  They totally discard the fact that the quantities $\EE N_i(t)$ and $\EE R_t$ must be non-decreasing. Tighter, relevant lower bounds can be recovered using this monotonicity. 

\begin{theorem}\label{thm:LB_passive}
The regret of a sub-logarithmic algorithm with constants $C$, $C_0$ must verify
\begin{align*}
\EE R_T \geq \sum_{i=2}^K \frac{1}{2\Delta_i} r_T^{(i)}
\end{align*}
where
\begin{align*}
r_T^{(i)} &=\log(\frac{T\Delta_i^2}{2C\log T\sum_{j\neq i}\frac{\Delta_i}{\Delta_i+\Delta_j}})
{+} \eta_i(T) {-} 2 \epsilon p^{(i)} \Delta_i^2 T
\end{align*}
if $T\leq 1/(2 \epsilon p^{(i)} \Delta_i^2)$ and otherwise
\begin{align*}
r_T^{(i)}&= \Bigg[ \log\left(\frac{1}{\epsilon}\frac{1}{4C p^{(i)}\sum_{j\neq i}\frac{\Delta_i}{\Delta_i+\Delta_j}}\right) \\
& \quad - \log\log(\frac{1}{2\epsilon p^{(i)}\Delta_i^2}) + \eta_i(\frac{1}{2\epsilon p^{(i)}\Delta_i^2}) - 1 \Bigg] \: .
\end{align*}
The function $\eta_i(T)$ goes to zero in $O(1/\log T)$. See appendix for details.
\end{theorem}

Theorem~\ref{thm:LB_passive} correctly reports a lower bound increasing with the horizon. It shows a transition from a $O(\log T)$ optimal regret for $T \ll 1/(2 \epsilon p^{(i)} \Delta_i^2)$ to a finite regret function of $\epsilon$ when $T$ gets bigger. According to Theorem~\ref{thm:LB_passive}, the correct dependency in $\epsilon$ in  the regret should be in $O(\log(1/\epsilon))$, not $O(1/\epsilon)$ as seen for the naive FTL-robin algorithm.

We can also wonder what is the most favorable passive setting. Simple computations show that  free observations should be drawn according to the probability vector $(p_\star^{(1)},\ldots,p_\star^{(K)})$ where $p_\star^{(i)}$ is proportional to $\frac{1}{\Delta_i}$ (here, we actually ignore the $\log\log$ and $\eta$ terms of Theorem~\ref{thm:LB_passive}), leading to a lowest lower bound 
\begin{align*}
\EE_1R_T
&\geq \sum_{i=2}^K\frac{1}{2\Delta_i} \log\left(\frac{1}{\epsilon}\frac{\sum_{j=2}^K\frac{1}{\Delta_j}}{4C \sum_{j\neq i}\frac{1}{\Delta_i+\Delta_j}}\right) + \alpha \\
&\geq \sum_{i=2}^K\frac{1}{2\Delta_i} \log\left(\frac{1}{4C\epsilon}\right) + \alpha \: ,
\end{align*}
where $\alpha$ regroups the $\log\log$ and $\eta$ terms in theorem~\ref{thm:LB_passive}. This lower bound shows in particular that when all sub-optimal arms have the same gap, the optimal sample distribution is uniform and the lower bound is of order $\frac{K}{\Delta}\log(\frac{1}{\epsilon})$ .

\subsection{ACTIVE OBSERVER}

An active observer has the possibility to chose the weights $p_t^{(i)}$ at each stage $t\leq T$, potentially achieving a much better distribution of the free observations up to stage $T$ than any static distribution.  As before, standard techniques give the following lower bound.

\begin{lemma}\label{lemma:LB_active_simple}
The regret of a sub-logarithmic algorithm with constants $C$, $C_0$ verifies
\begin{align*}
\mathbb{E}R_T \geq \sum_{i=2}^k \frac{h_i(T)}{2\Delta_i} - \Delta_k(\epsilon T - \sum_{j>k} \frac{h_j(T)}{2\Delta_j^2}) \: ,
\end{align*}
where $k = \min \{ i\in\{2,\ldots,K\} \: : \: \sum_{j>i} \frac{h_j(T)}{2\Delta_j^2} \leq \epsilon T \}$.
\end{lemma}
The structure of the solution to the optimization problem in this case is again educational: an optimal algorithm presented with a given amount of free observations would spend them at the beginning,   before costly pulls, and will spend them on the worst arms. This intuition drove the construction of algorithms for active  observer in section~\ref{SE:Upper}:

First gather free observations, ideally accordingly to the proportion $(p_\star^{(1)}, \ldots,p_\star^{(K)})$ then discards arms for which enough information were gathered, and use a standard optimal bandit algorithm on the remaining ones. 

As in the passive observer case, although this lower bound can be meaningful for small horizon $T$, it becomes void for larger horizons. A better lower bound using the monotony of the number of pulls and of the regret is provided in the next theorem.

\begin{theorem}\label{thm:LB_active}
For $k\in\{2, K-1\}$ let $t_k = \max\{t\geq 1 \: : \: \sum_{j=k+1}^K \frac{h_j(t)}{2\Delta_j^2} > \epsilon t \}$.
The regret of any active sub-logarithmic algorithm with constants $C$, $C_0$ verifies
\begin{align*}
\EE R_T &\geq \max_{k:t_k \leq T} \sum_{i=2}^k \frac{1}{\Delta_i}\Bigg[ \log(\frac{1}{\epsilon}\frac{\sum_{j=k+1}^K\frac{\Delta_i^2}{\Delta_j^2}}{4C \sum_{j\neq i} \frac{\Delta_i}{\Delta_i+\Delta_j}})\\
&- \log\log(\frac{1}{\epsilon}\sum_{j=k+1}^K\frac{1}{2\Delta_j^2})+ \eta(\frac{1}{\epsilon}\sum_{j=k+1}^K\frac{1}{2\Delta_j^2})\Bigg] \: .
\end{align*}
\end{theorem}

When all gaps are equal to the same value $\Delta>0$, the leading term of this lower bound is of the form
\begin{align*}
\max_{k:t_k \leq T} \frac{k-1}{\Delta}\log(\frac{1}{\epsilon}\frac{K-k}{K}) \: .
\end{align*}

In particular, this result states that as $T$ goes to infinity, the regret is asymptotically lower bounded by $\frac{K-1}{\Delta}\Big[\log(\frac{1}{\epsilon})-\log\log(\frac{e}{\epsilon})\Big]$.

\section{ALGORITHMS AND UPPER-BOUNDS}\label{SE:Upper}

In this section, we exhibit algorithms matching the lower bounds derived in the previous section,  up to $\log\log(\cdot)$ terms, showing that they indeed represent accurately the problem complexity.

\subsection{PASSIVE OBSERVER}\label{SE:UpperPassive}

A passive observer does not get to choose the arms on which   free information is gained. As in the classical stochastic multi-armed bandit, the only decision is therefore which arm to pull. It is then natural to extend known algorithms by taking into account all observations from both provenances.

As  UCB pulls the arm with maximal index $\overline{X}_t^{(i)} + \sqrt{\frac{6\log t}{N_i(t)}}$, we extend it by using all available observations both in the empirical mean and exploration term. Algorithm~\ref{algo:ucb_passive} pulls $i_t = \arg \max_{i} \overline{X}_t^{(i)} + \sqrt{\frac{6\log t}{O_i(t)}}$.

\begin{algorithm}
\caption{UCB with passive observations.}
\label{algo:ucb_passive}
\begin{algorithmic}[l]
  \State Pull each arm once.
  \Loop
  	  : at stage $t$,
  	  \State $i_t = \arg \max_{i} \overline{X}_t^{(i)} + \sqrt{\frac{6\log t}{O_i(t)}}$
      \State Pull arm $i_t$, observe $X_t^{(i)}$.
      \State If $Z_t = 1$, sample $f_t$ and observe $X_t^{(f_t)}$.
      \State Update $\overline{X}_t$, $N_i(t)$, $F_i(t)$, $O_i(t) = N_i(t)+F_i(t)$.
  \EndLoop
\end{algorithmic}
\end{algorithm}

\begin{theorem}\label{thm:ucb_passive}
Consider the static passive observer case, where $f_t$ follows the categorical distribution with parameters $(p^{(1)},\ldots,p^{(K)})$ and the probability of getting a free observation is $\epsilon\in(0,1]$ for all stages $t\geq 1$. 

Then the regret of ucb verifies both
\begin{align*}
\EE R_T &\leq \sum_{i=2}^K\frac{24}{\Delta_i} \log T \: ,\\
\mbox{and}&\\
\EE R_T &\leq  \sum_{i=2}^K \frac{24}{\Delta_i}\log\frac{50}{\epsilon p^{(i)}}\\
 &\qquad +  \sum_{i=2}^K \frac{24}{\Delta_i}\max\left\{\log\frac{1}{e\Delta_i^2} ,\log\log\frac{20}{\epsilon p^{(i)}} \right\} \: . 
\end{align*}
\end{theorem}

Hence UCB with passive observations recovers the $\log(\frac{1}{\epsilon})$ dependency in $\epsilon$, up to a doubly logarithmic term when $\epsilon p^{(i)}$ is small compared to the squared gaps. When the dominant term in this maximum is $\log\frac{1}{e\Delta_i^2}$, the regret due to arm $i$ has the form $\frac{1}{\Delta_i}\log\frac{1}{\epsilon p^{(i)} \Delta_i^2}$, which is sub-optimal with respect to $\Delta_i$ (see Theorem~\ref{thm:LB_passive}). This is due to the sub-optimality of UCB itself: while the regret of UCB on a bandit problem is $O(\sum_{i=2}^K \frac{\log T}{\Delta_i})$, other algorithms of the same family like UCB2 \citep{auer2002finite}, Improved-UCB, \citep{auer2010ucb} or MOSS \citep{audibert2009minimax, degenne2016anytime} get an improved regret of order $O(\sum_{i=2}^K \frac{\log (T\Delta_i^2)}{\Delta_i})$.

The dependency in $\log(\frac{1}{\epsilon})$ means that $\epsilon$ as small as $\frac{1}{T}$ gives useful information to a learner. Obviously there is no gain to be had if $\epsilon < \frac{1}{T}$, as there is in average less than one additional observation before $T$, but few more free observations are enough to improve the regret.

\subsection{ACTIVE OBSERVER} \label{SE:UpperActive}

While a uniform allocation of the free observations over the arms gets the right $\log(\frac{1}{\epsilon})$ dependency in $\epsilon$, having the choice of the arm which will be observed allows an algorithm to get the right dependency in the parameters of the bandit problem. In the active setting, the algorithm can choose freely which of the $[K]$ arms will get an additional observation, when such an observation is available.

To devise an algorithm taking advantage of this possibility, we try to mimic the lower bound for fixed stage, as in Lemma~\ref{lemma:LB_active_simple}. A good algorithm should use the available free observations first to discard the worse arms, before using costly pulls only on the remaining arms.

We introduce an algorithm combining two subroutines: an Explore-Then-Commit (ETC) \citep{even2006action, perchet2013multi} algorithm on the free observations is used to narrow the set of arms which need to be pulled and an algorithm of the UCB family is used on this set. As we seek for optimality with respect to the problem parameters we use OCUCB-n \citep{lattimore2016regret}, which is the UCB-type algorithm closest to it. ETC is described in Algorithm~\ref{algo:etc}. OCUCB-n with parameters $\eta>1$ and $\rho\in[1/2,1]$ pulls at stage $t\in\N^*$ the arm with maximal index
\begin{align*}
& \qquad \qquad \overline{X}_t^{(i)} + \sqrt{\frac{2\eta \log B_{t-1}^{(i)}}{N_i(t)}}\\
\mbox{where }&\\
 B_{t-1}^{(i)} &= \max\Bigg\{e,\log(t),
\frac{t\log t}{\sum_{i=1}^K \min\{N_i, N_j^\rho N_i^{1-\rho}\}} \Bigg\}
\end{align*}
where $N_i$ is a shorthand notation for $N_i(t)$.

The main algorithm use a succession of epochs. In epoch number $m\in\N$, the ETC subroutine collects (free) information on all the arms in $[K]$, while OCUCB-n pulls arms in an available subset of the arms $S_m$. At the end of epoch $m$, the free observations gathered are used to discard arms from $[K]$ which are not optimal with high enough confidence, forming $S_{m+1}$. There is a finite $m_i\in\N$ depending on $\epsilon$ and the gaps such that with high probability, $i\notin S_m$ for $m>m_i$, hence arm $i$ contributes to the regret only up to epoch $m_i$ and the regret is finite.

\begin{algorithm}
\caption{Active Algorithm.}
\label{algo:active}
\begin{algorithmic}[l]
  \Require parameters $\rho \in [1/2,1]$, $\alpha \geq 1, \eta>1$.
  \State Initialize $S_0=[K]$.
  \Loop
  	  : at epoch $m$, with duration $d_m=2^{2^m}$,
	  \State Pull arms according to OCUCB-n with parameters $\eta$ and $\rho$ on $S_m$,
	  \State Use free observations according to ETC with parameter $\alpha$ and horizon $T=d_{m+1}^{3/2}\log d_{m+1}$.
	  \State Set $S_{m+1}$ to the set returned by ETC.
  \EndLoop
\end{algorithmic}
\end{algorithm}

\begin{algorithm}
\caption{Explore-Then-Commit}
\label{algo:etc}
\begin{algorithmic}[l]
  \Require parameter $\alpha \geq 1$, horizon $T\in \N^*$.
  \State Initialize $s=0$, $S = [K]$.
  \Loop
      \State Observe all arms in $S$.
	  \State Discard from $S$ any arm $i$ such that\\
	  $\hat{\mu}_s^{(i)} + \sqrt{\frac{2\alpha}{s}\log(\frac{T}{s})}< \max_{j\in S}\hat{\mu}_s^{(j)} - \sqrt{\frac{2\alpha}{s}\log(\frac{T}{s})} \: .$
	  \State $s \leftarrow s+1$.
  \EndLoop\\
  \Return $S$.
\end{algorithmic}
\end{algorithm}

In order to write a regret upper bound for our active algorithm, we introduce quantities $H_{i,\rho}$ for $i\in\{2,\ldots,K\}$ and $\rho\in[1/2,1]$,
\begin{align*}
H_{i,\rho} &= \frac{i}{\Delta_i^2} + \sum_{j=i+1}^K \frac{1}{\Delta_i^{2(1-\rho)}\Delta_j^{2\rho}}\: .
\end{align*}
These constants transcribe the difficulty of the problem. A number of observations of order $\frac{1}{\epsilon}H_{i,1}$ will be necessary for ETC to eliminate arm $i$ with high confidence.

\begin{theorem}\label{thm:regret_active}
The regret of the active algorithm~\ref{algo:active} with parameters $\rho\in[1/2,1]$ and $\alpha = 1$ on problems with rewards in $[0,1]$ is
\begin{align*}
\EE R_T & \leq
C_\eta\sum_{i=2}^{K} \frac{4}{\Delta_i}\max \left\{\log(\frac{1}{\epsilon}),\log \sqrt{H_{i,\rho}}\right\}\\
&\qquad + 51K +O\left(\sum_{i=2}^K\frac{1}{\Delta_i}(\log\log \frac{H_{i,1}}{\epsilon})^2\right)
\end{align*}
with $C_\eta$ a constant that depends only on $\eta$ (see \citep{lattimore2016regret} for details on $C_\eta$).
\end{theorem}

Our analysis of Explore-Then-Commit relies on a new maximal concentration inequality which can be of independent interest.
\begin{lemma}
Let $Z_t$ be a $\sigma^2$-sub-Gaussian martingale difference sequence then, for every $\delta \in (0,0.2]$ and every integers $T \in \N^*$,
\begin{align*}
\PP\left\{ \exists t \leq T, \overline{Z}_t \geq \sqrt{\frac{2\sigma^2}{t}\log(\frac{T}{\delta t})} \right\} \leq 6\delta \sqrt{\log(\frac{1}{\delta})} \: .
\end{align*}
Asymptotically, we obtain
\begin{align*}
\limsup_{\delta \to 0} \frac{\PP\left\{ \exists t \leq T, \overline{Z}_t \geq \sqrt{\frac{2\sigma^2}{t}\log(\frac{T}{\delta t})} \right\}}{\delta \sqrt{\log(\frac{1}{\delta})}} \leq \sqrt{e/8}  .
\end{align*}
This value is $\sqrt{e/8} \approx 0.6$.
\end{lemma}

\subsubsection{Heuristics  and Influence of $\epsilon$}\label{SE:Heuristics}

Besides the algorithm already discussed, we also experimented on the following heuristic: choose a bandit algorithm of the UCB family, which pulls the arm with a maximal index; use it to pull the arm with maximal index and if an observation is available, observe the second maximal arm. We provide no regret analysis for this heuristic but study its performance in the experimental section.

\medskip

Concerning the dependency in $\varepsilon$, we can make the following interesting remark. To simplify notations, we will assume that all arms have the same gap $\Delta$ and we remove constants for this analysis. With these simplifications, we proved that  regret at stage $T$ is of the order of  $R_T \simeq \frac{K}{\Delta}\log(\frac{1}{\epsilon})$. Obviously, if $\epsilon$ is almost equal to $0$, this upper-bound is void and the algorithm should not depend on the free observations. One might ask what is the threshold at which free informations become relevant at stage $T$.

Notice that standard information theory arguments yield that if $\epsilon T \leq \frac{K}{2\Delta^2}$, and even if the free observations were gathered at the begining of the problem, only $\frac{K}{2}$ arms could be removed (with high probability) from the set of possible optimal arms. Hence these free information are not useful for at least $K/2$ arms and regret will have to scale as $\frac{K}{2\Delta}\log(\frac{2T\Delta^2}{K})$, the optimal rate for the bandit problem with $K/2$ arms with equal gaps $\Delta$.

On the other hand, if  $\epsilon T \geq \frac{K}{2\Delta^2}$, then (up to multiplicative constant), $\frac{K}{\Delta}\log(\frac{1}{\epsilon})$ dominates $\frac{K}{\Delta}\log(\frac{T\Delta^2}{K})$. As a consequence, the relevant threshold for the probability of free observations after $T$ stages is 
$$
\varepsilon^* =  \frac{1}{T}\frac{K}{\Delta^2}\, .
$$

\section{EXPERIMENTS}\label{SE:Expe}

All experiments are performed with Gaussian rewards with unit variance.

\paragraph{Influence of $\epsilon$.} 
The goal of this first experiment is to confirm the scaling of the regret with $\epsilon$. That is to say, the regret scales  with $ \sum_{i : \Delta_{i} > 0} \frac{1}{\Delta_{i}}\log(\frac{1}{\epsilon \Delta_{i}^{2}})$. The experiment is performed with a passive observer with either a uniform distribution or the optimal one, as defined in Section \ref{SE:LowerPassive}. To do so, the experiment is performed in the passive setting associated with a uniform distribution and the optimal one, as defined in Section \ref{SE:UpperPassive}. Also, when free observations are scarce, $\epsilon \sim \frac{1}{T}$, the average number of those is approximately $1$ during the experience. Therefore, the regret is similar to the one suffered by an UCB algorithm in a classic multi-armed bandit setting, a behaviour captured by the function $f$. On Figure \ref{Influence_eps_1} and \ref{Influence_eps_2}, experiments are run on four Gaussian arms with expectations $2$, $1.8$, $0.5$, $0.2$, the error bars are quantile at $10\%$ and $90\%$.
 
\begin{figure}[!ht]
\centering
\includegraphics[width = 0.45\textwidth]{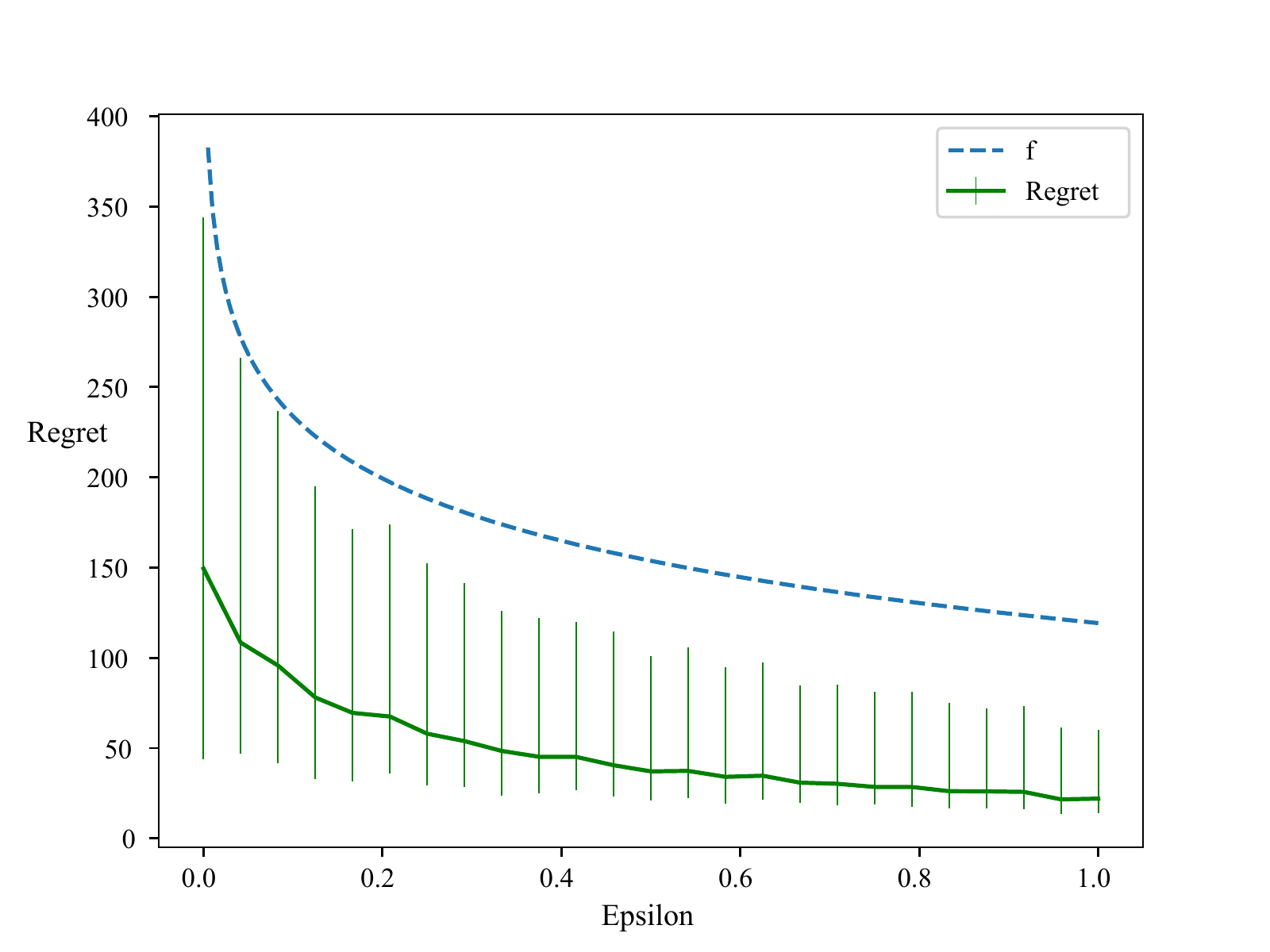}
\caption{\label{Influence_eps_1} Dependence on $\epsilon$ of the regret of UCB as passive observer, with a uniform distribution of the free observations, averaged over $300$ runs.}
\end{figure}

\begin{figure}[!ht]
\centering
\includegraphics[width = 0.45\textwidth]{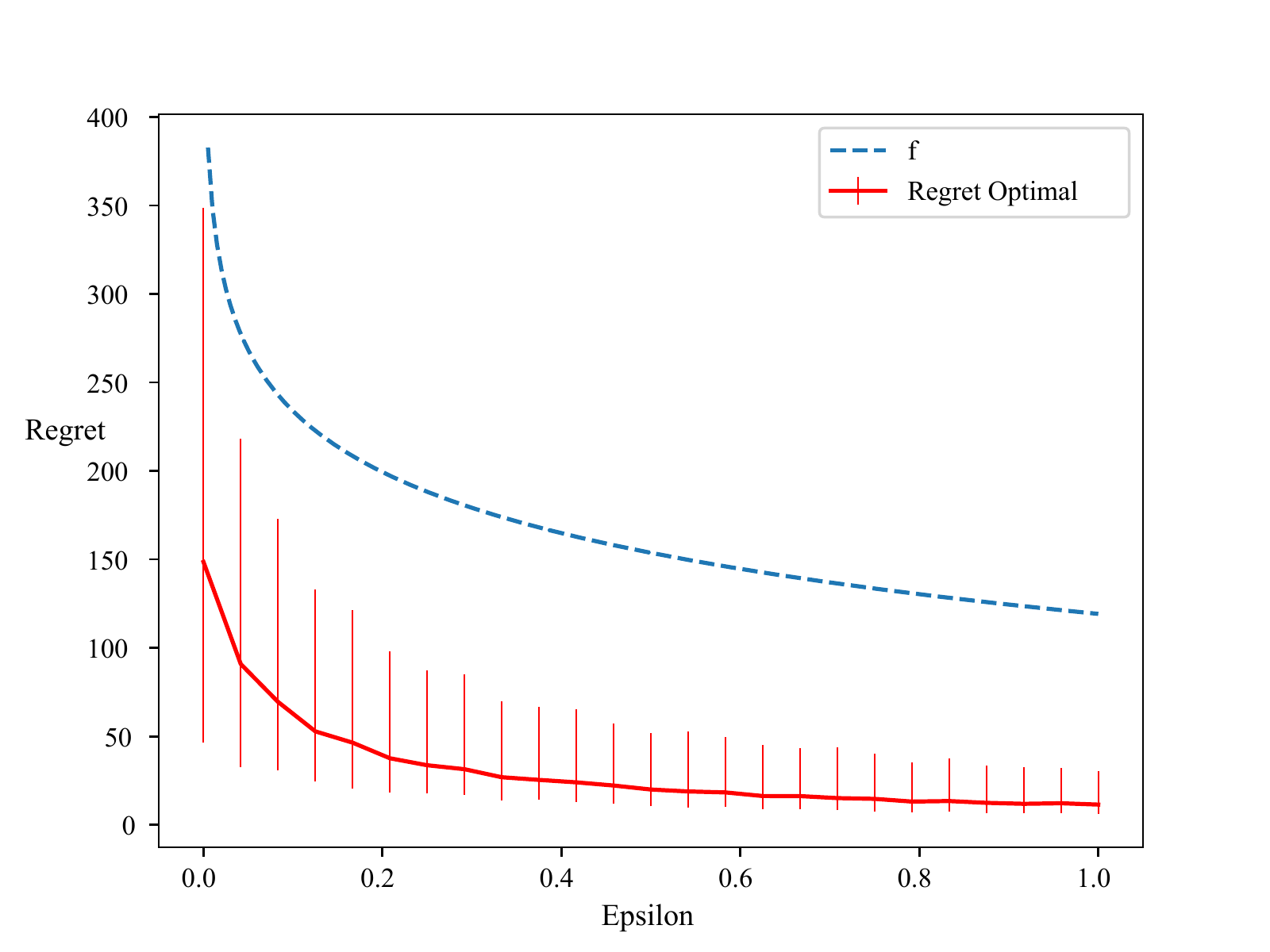}
\caption{\label{Influence_eps_2} Dependence on $\epsilon$ of the regret of UCB as passive observer, with the optimal distribution of the free observations, averaged over $300$ runs.}
\end{figure}

\paragraph{Passive Observer: optimal sampling distribution.}

This second experiment illustrates the induced regret in the passive setting with a probability distribution $p^{(i)} = \frac{1}{\Delta_{i}}$. This distribution is considered to be optimal because, as mentioned in Section  \ref{SE:LowerPassive}, it achieves the lowest lower bound. It also suggests a paradigm for algorithms in the active setting i.e sampling freely as much as possible the arm with the lowest $\Delta_{i}$. A way to do so is to run an UCB type algorithm to choose which arm to pull, and use another UCB type algorithm on other arms to determine which will be observed if a free observation is available. The results of this type of policy is presented in the next paragraph.

The experiment is run on the same set of arms as previously with a uniform distribution, the optimal distribution and a suboptimal one such that $p^{(i)}= \frac{1}{\Delta_{i}^{2}}$, referred as SubOptimal in Figure \ref{Opti_v_Unif}. Color filled regions are $25\%$ and $75\%$ quantiles.

\begin{figure}[!ht]
\centering
\includegraphics[width = 0.45\textwidth]{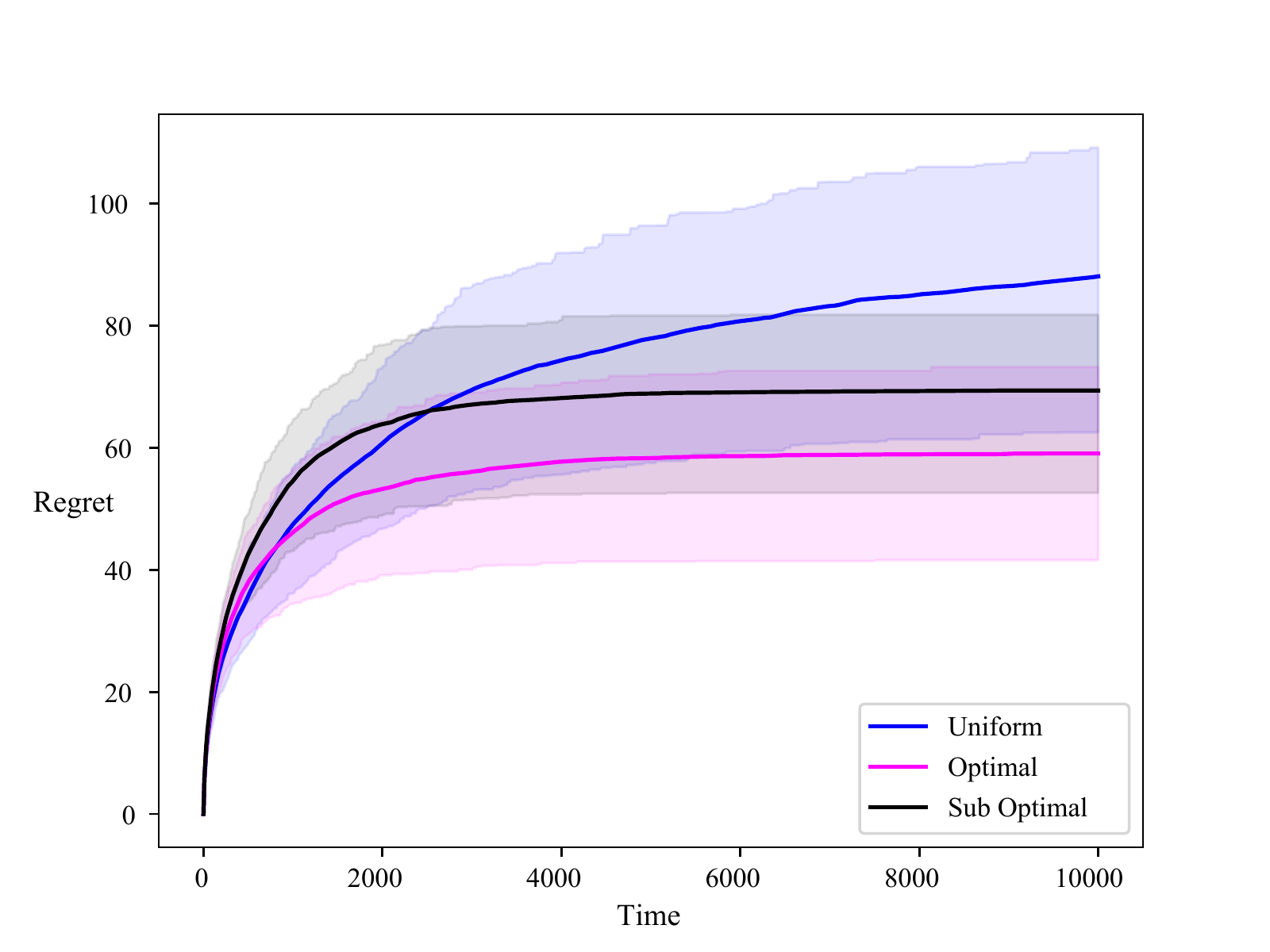}
\caption{\label{Opti_v_Unif}  Regret averaged over $300$ runs}
\end{figure}

\paragraph{Active Observer: comparison of algorithms.}

This subsection is dedicated to the comparaison of algorithms introduced earlier : UCB1-Double, ETC-OCUCB and ETC-OCUCB-2. \\
UCB1-Double uses a UCB algorithm and select the free observation as the second index maximising arm. The optimal allocation in the passive setting samples better arms more often, therefore we use the free observation to sample the arm next to optimal (according to its UCB index).  The second algorithm, referred to as ETC-OCUCB, is the algorithm studied in the above section. In particular, its ETC subroutine checks for potentially removable arms every $C|S|$ pulls, with $C$ a fixed parameter and $S$ the set of  currently active arm. Finally, the algorithm referred to as ETC-OCUCB-2 is a variant of ETC-OCUCB where elimination checks are made every $2^{k}$ stages, thus behaving less aggressively than ETC-OCUCB. In addition, we introduced in this experiment a parameter $p$ so that the epoch length is $d_{m} = p^{p^{m}}$ in ETC-OCUCB. This enables us to adapt the growth of epochs to the horizon, here $T = 10^{4}$. Other parameters are : $\alpha = 1$, $\rho = \frac{1}{2}$, $\eta = 2$ and $C = 10$. \\
The experiments is run on five Gaussian arms with expectations $2$, $1.8$, $1.5$, $1$ and $0.5$. Color filled regions are $25\%$ and $75\%$ quantiles.

\begin{figure}[!ht]
\centering
\includegraphics[width = 0.45\textwidth]{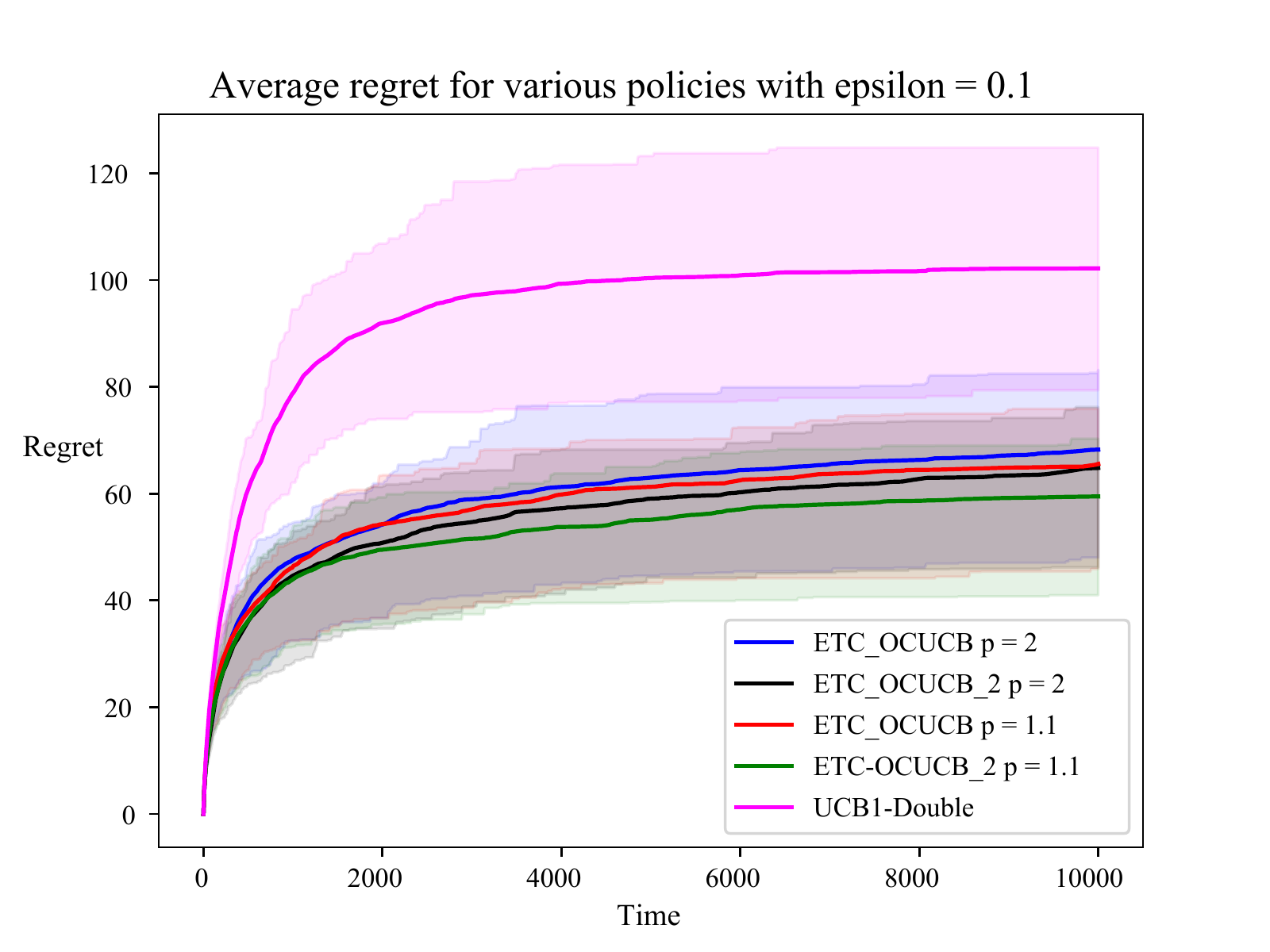}
\caption{\label{comp-algo} Regret for $\epsilon = 0.1$ averaged over $100$ runs}
\end{figure}

Figure \ref{comp-algo} illustrates that:

\begin{itemize}
\item UCB1-Double reaches rapidly its final regret value after a logarithmic exploration phase where informations are gathered so that the policy doesn't pull an other suboptimal arm after this phase.
\item ETC-OCUCB and ETC-OCUCB-2 algorithms have similar performances and the parameter $p$ offers a control how often the set of active arms is updated which offers a slight performance increase for lower $p$. 
\end{itemize}

ETC-OCUCB and ETC-OCUCB-2 maintain two distinct tracks of rewards, one for rewards obtained after pulling an arm and the other for rewards after sampling freely an arm. Therefore, it may be possible to increase their performance by using both sources of information in both subroutines. In the Figure below, these variants are referred as ETC-OCUCB-all-info and ETC-OCUCB-all-info-2.

\begin{figure}[!ht]
\centering
\includegraphics[width = 0.45\textwidth]{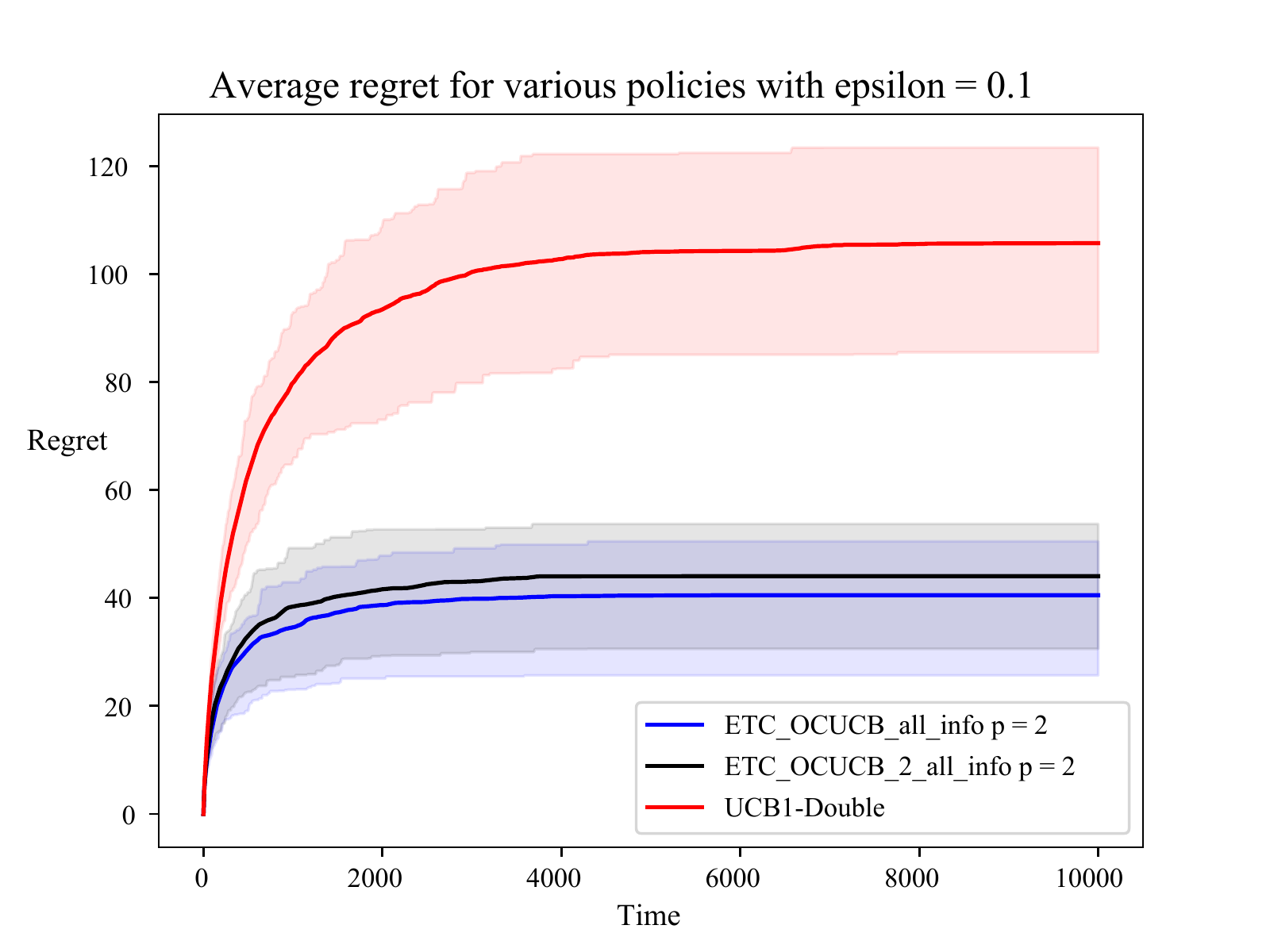}
\caption{\label{} Regret for $\epsilon = 0.1$ averaged over $300$ runs for $p=2$}
\end{figure}

This simple modification provides a clear improvement whether for the final regret or the speed at which this value is reached.

\section{CONCLUSION}

We analysed the multi-armed bandit problem with just a few extra free information. Interestingly, as the regret is uniformly bounded in time, standard lower bounds are void. However, a careful analysis allowed us to exhibit non-trivial guarantee that no reasonable algorithm can out-perform and we finally provided an optimal algorithm, whose regret matches the lower bound up to doubly logarithmic terms.

 We would like to finally emphasize that our algorithm can be used even if the $\varepsilon T$ observations are not free. Since we used ETC on these observations, we get that our algorithm has a regret smaller (discarding multiplicative constants and $\log\log$ terms) than
 $$
 \sum_{i=2}^K \frac{\log(\varepsilon T \Delta_i^2)}{\Delta_i} + \sum_{i=2}^K \frac{\log(1/\varepsilon)}{\Delta_i}
 $$
 where the first term is the guarantee of ETC on $\varepsilon T$ samples, and the second one is the guarantee of our algorithm with ``free'' observations. As a consequence, no matter the value of $\varepsilon$ (as long as the $\log\log$ terms do not become dominant), its  dependency vanishes, and we recover the expected performance of ETC.

\subsubsection*{Acknowledgements}
V. Perchet has benefited from the support of the ANR (grant n.ANR-13- JS01-0004-01), of the FMJH Program Gaspard Monge in optimization and operations research (supported in part by EDF), from the Labex LMH and from the CNRS, PEPS project Lacreme.

\newpage


\nocite{yu2009piecewise}

\bibliographystyle{abbrvnat}
\setcitestyle{authoryear,open={[},close={]}}
\bibliography{bibli}

\clearpage
\appendix

\section{LOWER BOUND PROOF}

Consider a bandit problem with Gaussian arms with distributions $\nu^{(i)}=\mathcal{N}(\mu^{(i)},1)$ with $\mu^{(1)}>\mu^{(2)}\geq\ldots\geq\mu^{(K)}$ , denoted by problem 1. We define $K-1$ other bandit problems in which an arm is changed to bring it above the optimal arm. Formally, problem $i$ with distributions $(\nu_i^{(1)},\ldots,\nu_i^{(K)})$ is such that $\nu_i^{(j)}=\nu^{(j)}$ for all $j\in[K]\setminus\{i\}$, and $\nu_i^{(i)} = \mathcal{N}(\mu^{(1)} + \Delta, 1)$ . The distributions of $(Z_t)_{t\geq 1}$ are the same in all problems.

Let $I_t = (i_1, f_1,Z_1, X_1^{(i_1)},\ldots,i_T,f_T,Z_T, X_i^{(i_T)})$. The Kullback-Leibler divergence between the observations up to time $T$ coming from problem 1 and problem $i\neq 1$ is
\begin{align*}
\KL(\mathbb{P}_1^{I_T}	, \mathbb{P}_i^{I_T}) = \EE_1 O_i(T) \frac{(\Delta+\Delta_i)^2}{2}
\end{align*}
By showing lower bounds on this divergence, we prove constraints on $\EE_1 O_i(T)$, leading to a lower bound on the regret. By using the principle of contraction of entropy \citep{garivier2016explore}, we can relate the divergence between the observations in the two problems to the Kullback-Leibler divergence between Bernoulli variables. Let $\kl(a,b)$ denote the Kullback-Leibler divergence between Bernoulli distributions with parameters $a$ and $b$. $\kl(a,b) = a\log\frac{a}{b} + (1-a)\log\frac{1-a}{1-b}$ .
\begin{align*}
\KL(\mathbb{P}_1^{I_T}	, \mathbb{P}_i^{I_t})
&\geq \kl(\mathbb{E}_1\frac{N_1(T)}{T}, \mathbb{E}_i\frac{N_1(T)}{T})\\
&\geq \mathbb{E}_1\frac{N_1(T)}{T}\log\frac{1}{\mathbb{E}_i\frac{N_1(T)}{T}} - \log(2)
\end{align*}
where we used that $\kl(p,q)\geq p\log(1/q)-\log 2$ .

The Expectations of the number of pulls will be bounded through the hypothesis of sub-logarithmic regret: as the regret must be low, the number of pulls of sub-optimal arms must also be low. The algorithm is sub-logarithmic with constants $C$, $C_0$ at all stages $T\in \mathbb{N}$ . i.e. on all multi-armed bandit problems, for all stages $T$,  $\mathbb{E}R_T \leq C_0\sum_{i=2}^K\Delta_i + C\sum_{i:\Delta_i>0}\frac{\log T}{\Delta_i}$ . let $C_\Delta = \sum_{i:\Delta_i>0}\frac{1}{\Delta_i}$ and $C_K = C_0\sum_{i=2}^K\Delta_i$.
\begin{align*}
\EE_1\frac{N_1(T)}{T}
&= 1 - \sum_{i\neq 1} \EE_1\frac{N_i(T)}{T}\\
&\geq 1 -\frac{1}{T}(C_K + CC_\Delta\log T) \: ,\\
\EE_i\frac{N_1(T)}{T}
&\leq \frac{1}{T\Delta}(C_K + C\sum_{j\neq i}\frac{1}{\Delta + \Delta_j}\log T) \: .
\end{align*}
We obtain finally the constraint
\begin{align*}
\EE_1 O_i(T) &\geq \frac{2}{(\Delta + \Delta_i)^2}h_i(T) \: ,
\end{align*}
where
\begin{align*}
h_i(T) &= \log(\frac{T\Delta^2}{2C\log T\sum_{j\neq i}\frac{\Delta}{\Delta+\Delta_j}}) +\eta_i(T) \: ,\\
\eta_i(T) &= -\frac{1}{T}(C_K + CC_\Delta\log T)\\
&\quad \times\log(\frac{T\Delta^2}{C_K\Delta + C\log T\sum_{j\neq i}\frac{\Delta}{\Delta+\Delta_j}})\\
&- \log(1 + \frac{C_K}{C\log T \sum_{j\neq i}\frac{1}{\Delta + \Delta_j}}) \: .
\end{align*}

\subsection{Properties of $h_i$ and $\eta_i$.}

The function $h_i$ is increasing over $[e,+\infty)$. Its derivative verify
\begin{align*}
h_i'(t) \geq \frac{1}{t}\left(1 - \frac{1}{\log t + \frac{C_K}{C\sum_{j\neq i}\frac{1}{\Delta + \Delta_j}}}\right) \: .
\end{align*}
If $t\geq e^2$, then $h_i'(t) \geq 1/(2t)$.

Let $f_i(t) = \frac{2h_i(t)}{(\Delta+\Delta_i)^2} - \epsilon p^{(i)} t$. then for $t\geq e^2$, $f_i'(t) \geq \frac{1}{(\Delta+\Delta_i)^2 t} - \epsilon p^{(i)} $, such that $f_i(t)$ is increasing over $[e^2, 1/(\epsilon p^{(i)} (\Delta+\Delta_i)^2)]$.

\subsection{Passive Static Setting}

In this setting, $\EE_1 F_i(T) = \epsilon T p^{(i)}$ for all $i\in[K]$ (with $\sum_{i=1}^Kp^{(i)}=1$), such that $\EE_1 O_i(T) = \EE_1 N_i(T) + \epsilon T p^{(i)}$ . Using the constraint on $O_i(T)$, we deduce that the regret of a sub-logarithmic algorithm must be bigger than the solution of the optimization problem
\begin{align*}
\mbox{minimize in $n$: }\quad & \sum_{i=2}^K n_i \Delta_i \\
\mbox{subject to }\quad & \forall i\geq 2 ,\: n_i \geq \frac{2h_i(T)}{(\Delta+\Delta_i)^2} - \epsilon T p^{(i)} \: ,\\
& n\succeq 0 \: .
\end{align*}
The solution is given by $n_i =\max(0,\: \frac{2h_i(T)}{(\Delta+\Delta_i)^2} - \epsilon Tp^{(i)})$ . We see that for $T$ big enough, the lower bound is 0. This does not reflect the problem at hand since some regret is unavoidable at the beginning, when few free observations are available.

Since $\EE_1N_i(T)$ is non-decreasing, we can aggregate the constraints on this quantity up to stage $T$ to get the stronger constraint
\begin{align*}
\EE_1 N_i(T) &\geq \sup_{3\leq t \leq T} \left\{\frac{2h_i(t)}{(\Delta+\Delta_i)^2} - \epsilon tp^{(i)}\right\}
\end{align*}
3 is taken as the starting point for $t$ to ensure that $h_i$ is increasing.

\paragraph{Small horizon: $T\leq 1/(\epsilon p^{(i)} (\Delta + \Delta_i)^2)$.}
\begin{align*}
\EE N_i(T)
&\geq \frac{2h_i(T)}{(\Delta+\Delta_i)^2} - \epsilon Tp^{(i)}\\
&\geq \frac{1}{(\Delta+\Delta_i)^2}\left(2h_i(T) - 1\right)\: .
\end{align*}
\paragraph{Big horizon: $T\geq 1/(\epsilon p^{(i)} (\Delta + \Delta_i)^2)$.}
\begin{align*}
\EE_1 N_i(T)
&\geq \frac{1}{(\Delta + \Delta_i)^2}\left[2h_i(\frac{1}{\epsilon p^{(i)}(\Delta + \Delta_i)^2}) - 1\right] \: ,
\end{align*}
where this value is obtained by taking $t = \frac{1}{\epsilon p^{(i)}(\Delta+\Delta_i)^2}$.

In the construction of this lower bound, we can choose $\Delta$ separately for each arm. Using $\Delta = \Delta_i$ in each $h_i$, we get
\begin{align*}
\EE_1R_T
&\geq \sum_{i=2}^K\frac{1}{2\Delta_i} \Bigg[ \log\left(\frac{1}{\epsilon}\frac{1}{8C p^{(i)}\sum_{j\neq i}\frac{\Delta_i}{\Delta_i+\Delta_j}}\right) \\
& \quad - \log\log(\frac{1}{4\epsilon p^{(i)}\Delta_i^2}) + \eta_i(\frac{1}{4\epsilon p^{(i)}\Delta_i^2}) - \frac{1}{2} \Bigg] \: .
\end{align*}

\subsection{Active Setting}

Using the constraints, we obtain that the regret of any sub-logarithmic algorithm must verify that $\EE_1 R_T$ is bigger than the solution of the problem
\begin{align*}
\mbox{minimize in $n,f$: }\quad & \sum_{i=2}^K n_i \Delta_i \\
\mbox{subject to }\quad & \forall i\geq 2, \: n_i + f_i \geq \frac{2h_i(T)}{(\Delta+\Delta_i)^2} \: ,\\
& \sum_{i=2}^K f_i \leq \epsilon T \: , \: n\succeq 0 \: , \: f \succeq 0 \: .
\end{align*}

The solution of this optimization problem has the following structure: there exists a $\nu \in \{ 0, \Delta_2, \ldots, \Delta_K\}$ such that for all $i$ such that $\Delta_i < \nu$, $n_i = \frac{2h_i(T)}{(\Delta + \Delta_i)^2}$ and $f_i=0$ ; for all $j$ such that $\Delta_j>\nu$, $n_j=0$ and $f_j=\frac{2h_j(T)}{(\Delta + \Delta_j)^2}$ ; for the possible index $k$ with $\Delta_k = \nu$, $f_k = \epsilon T - \sum_{j>k}\frac{2h_j(T)}{(\Delta + \Delta_j)^2}$ and $n_k = \frac{2h_k(T)}{(\Delta+\Delta_k)^2} - f_k$ . That is, an optimal algorithm uses the free information on bad arms and uses the costly pulls on good arms. The optimal attainable expected regret is then
\begin{align*}
\mathbb{E}R_T \geq \sum_{i\leq k} \frac{2 \Delta_i h_i(T)}{(\Delta+\Delta_i)^2} - \Delta_k(\epsilon T - \sum_{j>k} \frac{2 h_j(T)}{(\Delta+\Delta_j)^2}) \: ,
\end{align*}
where $k = \min \{ i\in\{2,\ldots,K\} \: : \: \sum_{j>i} \frac{2 h_j(T)}{(\Delta+\Delta_j)^2} \leq \epsilon T \}$ .

\subsubsection{Increasing number of pulls}

The lower bounds for increasing stages $T$ show that free information should progressively replace pulls, starting from worse arms. For $T$ big enough the lower bound on $\EE R_T$ is 0. An optimal algorithm should somehow have used only free information to explore. This is impossible, since the algorithm doesn't know at first which arm is the best. The lower bound exhibits this behaviour because it is written for fixed $T$ and ignores that both $\EE N_i(t)$ and $\EE R_t$ must be non-decreasing. We can get a tighter lower bound by using this monotonicity.
\begin{align*}
\EE R_T
&\geq \max_{t\leq T}\sum_{i\leq k_t} \frac{h_i(t)}{2\Delta_i} - \Delta_{k_t}(\epsilon t - \sum_{j>k_t} \frac{h_j(t)}{2\Delta_j^2})\\
&\geq \max_{t\leq T}\sum_{i< k_t} \frac{h_i(t)}{2\Delta_i} \: .
\end{align*}
For $k\in\{2, \ldots, K\}$, let $t_k = \max\{t\geq 1 \: : \: \sum_{j>k} \frac{h_j(t)}{2\Delta_j^2} > \epsilon t \}$, such that $t\leq t_k \Leftrightarrow k_t>k$ . We can rewrite the lower bound on the regret to introduce these stages,
\begin{align*}
\EE R_T
\geq \max_{k:t_k \leq T} \sum_{i=2}^k \frac{h_i(t_k)}{\Delta_i}
\end{align*}
$t_k$ and $h(t_k)$ verify
\begin{align*}
t_k &\geq \frac{1}{\epsilon}\sum_{j=k+1}^K\frac{1}{2\Delta_j^2} \: ,\\
h_i(t_k) &\geq \log(\frac{1}{\epsilon}\frac{\sum_{j=k+1}^K\frac{\Delta_i^2}{\Delta_j^2}}{4C \sum_{j\neq i} \frac{\Delta_i}{\Delta_i+\Delta_j}})\\
&- \log\log(\frac{1}{\epsilon}\sum_{j=k+1}^K\frac{1}{2\Delta_j^2})+ \eta(\frac{1}{\epsilon}\sum_{j=k+1}^K\frac{1}{2\Delta_j^2}) \: .
\end{align*}
Let $B_{i,k}(\epsilon)$ be this lower bound for $h_i(t_k)$. Then
\begin{align*}
\EE R_T \geq \max_{k:t_k \leq T} \sum_{i=2}^k \frac{B_{i,k}(\epsilon)}{\Delta_i} \: .
\end{align*}

\subsubsection{Alternative lower bound}

Alternatively, the regret in the active setting is lower bounded by the regret of a passive setting in which, when $Z_t=1$, all arms get a free observation. 
The constraint on $\EE N_i(T)$ becomes
\begin{align*}
\EE N_i(T) \geq \EE N_i(t) \geq \frac{h_i(t)}{\Delta^2} - \epsilon t \: .
\end{align*}
From that get a lower bound as if $p_t^{(i)} = 1$ for all $t$ and $i$.

\begin{align*}
\EE_1R_T
&\geq \sum_{i=2}^K\frac{1}{2\Delta_i} \Bigg[ \log\left(\frac{1}{\epsilon}\frac{1}{4C\sum_{j\neq i}\frac{\Delta_i}{\Delta_i+\Delta_j}}\right) \\
& \quad - \log\log(\frac{1}{2\epsilon \Delta_i^2}) + \eta_i(\frac{1}{2\epsilon \Delta_i^2}) - 1 \Bigg] \: .
\end{align*}

When all gaps are equal, this is of order $\frac{K}{\Delta}\log(\frac{1}{\epsilon K})$ and the $K$ factor in the logarithm is suboptimal.

\section{UPPER BOUND}

\subsection{Generalities}

\subsubsection{Concentration of Sub-Gaussian Random Variables}

Recall that a process $\{Z_t\}_{t\geq 0}$ is a $\sigma^2$-sub-Gaussian martingale difference sequence if $\EE[Z_{t+1}|Z_1,\ldots,Z_t] =
0$ and $\log \EE[e^{\lambda Z_{t+1}}] \leq \frac{1}{2}\sigma^2\lambda^2$ for every $\lambda > 0$, $t \geq 0$.

\begin{lemma}
Let $Z_t$ be a $\sigma^2$-sub-Gaussian martingale difference sequence then, for every $\delta > 0$ and every integers $T_1 \leq T_2 \in \N^*$,
\begin{align*}
\PP\{\exists t\in[T_1,T_2],\: \overline{Z}_t \geq \sqrt{\frac{2\sigma^2}{t}\log(\frac{1}{\delta})\phi(\frac{T_2}{T_1})} \} \leq \delta \: .
\end{align*}
where the mapping $\phi(\cdot)$ is defined by $\phi(x) = \frac{1+x+2\sqrt{x}}{4\sqrt{x}}$ and it holds that $1 - \frac{(x-1)^2}{16} \leq \frac{1}{\phi(x)} \leq 1$ .
\end{lemma}

\begin{proof}
Using directly Hoeffding's inequality would not be that useful, instead, we shall go a step back in its classical proof. Since $t\mapsto \sqrt{t}$ is concave on $[T_1 , T_2]$ we can lower bound by the following linear form:
\begin{align*}
\sqrt{t} &\geq \frac{\sqrt{T_2} - \sqrt{T_1}}{T_2 - T_1} t + (\sqrt{T_1} - T_1 \frac{\sqrt{T_2} - \sqrt{T_1}}{T_2 - T_1})\\
&= \frac{t}{\sqrt{T_1} + \sqrt{T_2}} + \frac{\sqrt{T_1T_2}}{\sqrt{T_1} + \sqrt{T_2}} := \eta t + \lambda \: .
\end{align*}
As a consequence we obtain, for $\beta>0$, denoting $\alpha = \sqrt{8 \beta / \sigma^2 } \eta$,
\begin{align*}
 &\PP \left\{ \exists t\in[T_1,T_2], \overline{Z}_t \geq \sqrt{\frac{2\sigma^2}{t} \beta}\right\}\\
=   &\PP \left\{ \exists t\in[T_1,T_2], t\overline{Z}_t \geq \sqrt{2\sigma^2 \beta}\sqrt{t} \right\}\\
\leq &\PP \left\{ \exists t\in[T_1,T_2], t\overline{Z}_t \geq \sqrt{2\sigma^2 \beta}\eta t + \sqrt{2\sigma^2 \beta} \lambda \right\}\\
=   &\PP \left\{ \exists t\in[T_1,T_2], t\overline{Z}_t \geq \frac{\alpha \sigma^2}{2}t + \frac{1}{\alpha} 4\eta\lambda \beta \right\}\\
=   &\PP \left\{ \exists t\in[T_1,T_2], \alpha t\overline{Z}_t \geq \frac{\alpha^2 \sigma^2}{2}t +  4\eta\lambda \beta \right\}\\
&\leq \exp(-4 \eta \lambda \beta)
\end{align*}
where the last inequality is just a consequence of Doob and Markov inequalities. By definition of $\eta$ and $\lambda$, we always have
\begin{align*}
4\eta\lambda = \frac{4\sqrt{T_1T_2}}{T_1+T_2+2\sqrt{T_1T_2}} \leq 1 \: .
\end{align*}
Finally, if $T_2 = (1+\gamma)T_1$, then this ratio is independent of $T_1$ and
\begin{align*}
1 \geq 4\eta\lambda = \frac{4\sqrt{1+\gamma}}{2 + \gamma + 2\sqrt{1+\gamma}} \geq 1 - \frac{\gamma^2}{16} \: .
\end{align*}
Taking $\beta = \frac{\log(1/\delta)}{4\eta\lambda} = \frac{1}{4}\log(\frac{1}{\delta})\phi(\frac{T_2}{T_1})$ gives the concentration inequality.
\end{proof}

\begin{lemma}\label{lemma:martingale_concentration}
Let $Z_t$ be a $\sigma^2$-sub-Gaussian martingale difference sequence then, for every $\delta \in (0,0.2]$ and every integers $T \in \N^*$,
\begin{align*}
\PP\left\{ \exists t \leq T, \overline{Z}_t \geq \sqrt{\frac{2\sigma^2}{t}\log(\frac{T}{\delta t})} \right\} \leq 6\delta \sqrt{\log(\frac{1}{\delta})} \: .
\end{align*}
Asymptotically, we obtain
\begin{align*}
\limsup_{\delta \to 0} \frac{\PP\left\{ \exists t \leq T, \overline{Z}_t \geq \sqrt{\frac{2\sigma^2}{t}\log(\frac{T}{\delta t})} \right\}}{\delta \sqrt{\log(\frac{1}{\delta})}} \leq\sqrt{e/8}  \:  .
\end{align*}
and $\sqrt{e/8} \approx 0.6$.
\end{lemma}

\begin{proof}
Define $\epsilon_t = \sqrt{\frac{2\sigma^2}{t}\log(\frac{T}{\delta t})}$. We use the classical peeling argument with respect to the grid $\gamma = (1 + \eta)$, $\gamma^2 = (1 + \eta)^2$, ... Let $I_m$ be the interval $[\gamma^m,\gamma^{m+1}]$.
\begin{align*}
&\PP\left\{ \exists t \leq T, \overline{Z}_t \geq \epsilon_t \right\}\\
\leq & \sum_{m=0}^{\lfloor \log_\gamma(T)\rfloor} \PP\left\{ \exists t \in I_m, \sqrt{\frac{2\sigma^2}{t}\phi(\gamma) (1-\frac{\gamma^2}{16})\log(\frac{T}{\delta t}}) \right\}\\
\leq & \kern-1em\sum_{m=0}^{\lfloor \log_\gamma(T)\rfloor}\kern-1em \PP\left\{ \exists t \in I_m, \sqrt{\frac{2\sigma^2}{t}\phi(\gamma) (1{-}\frac{\gamma^2}{16})\log(\frac{T}{\delta \gamma^{m+1}}}) \right\}
\end{align*}
As a consequence
\begin{align*}
&\PP\left\{  \exists t \leq T, \overline{Z}_t \geq \epsilon_t  \right\}\\
\leq & \sum_{m=0}^{\lfloor \log_\gamma(T)\rfloor} \left( \frac{\delta \gamma^{m+1}}{T} \right)^{1-\eta^2/16}\\
=    & \left( \frac{\delta \gamma}{T} \right)^{1-\eta^2/16} \sum_{m=0}^{\lfloor \log_\gamma(T)\rfloor} \left( \gamma^{1-\eta^2/16} \right)^m\\
= & \left( \frac{\delta \gamma}{T} \right)^{1-\eta^2/16} \frac{\left( \gamma^{1-\eta^2/16} \right)^{\lfloor \log_\gamma(T)\rfloor} - 1}{\gamma^{1-\eta^2/16} - 1}\\
\leq & (\delta \gamma^2)^{1-\eta^2/16} \frac{1}{\gamma^{1-\eta^2/16} - 1}\\
\leq & \delta \delta^{-\eta^2/16} \frac{10}{\eta}
\end{align*}
as soon as $\eta \leq 2.4$ . Now the specific choice of $\eta = \sqrt{8/\log(1/\delta)}$, which is valid as soon as $\delta \leq 0.2$, gives
\begin{align*}
&\PP\left\{ \exists t \leq T, \overline{Z}_t \geq \sqrt{\frac{2\sigma^2}{t}\log(\frac{T}{\delta t})} \right\}\\
\leq & 5 \sqrt{\frac{e}{2}}\delta \sqrt{\log(1/\delta)} \leq 6\delta \sqrt{\log(1/\delta)} \: .
\end{align*}

\end{proof}

\subsubsection{Concentration of sums of binary variables}

\begin{lemma}\label{lemma:binary_concentration_variance}
Let $X_1,\ldots,X_n$ be independent Bernoulli random variables with means $(p_i)_{1\leq i \leq n}$ and let $S_n$ be their sum. Let $p = \frac{1}{n}\sum_{i=1}^n p_i$. For $\alpha>0$,
\begin{align*}
\PP\left\{ S_n - np \leq -n\alpha \right\} \leq \exp(-n\phi_p^*(p+\alpha)) \: ,
\end{align*}
where the application $\phi_p^*$ is defined for $x\in[0,1]$ by $\phi_p^*(x) = x\log\frac{x}{p} + (1-x)\log \frac{1-x}{1-p}$ . It $p\leq 1/2$ then
\begin{align*}
\PP\left\{ S_n - np \leq -n\alpha \right\} \leq \exp\left( -\frac{n\alpha^2}{2p(1-p)} \right) \: .
\end{align*}
\end{lemma}

\begin{proof}
The first inequality relates the sum of $n$ Bernoulli random variable to the case of a single Binomial variable with parameters $n$ and $p=\frac{1}{n}\sum_{i=1}^n p_i$ . It is an application of Chernov's method. For $\lambda>0$,
\begin{align*}
&\PP\left\{ S_n - np \leq -n\alpha \right\}\\
=&    \PP\left\{ (n-S_n) \geq n(1-p) + n\alpha \right\}\\
\leq& \PP\left\{ \exp(\lambda (n-S_n)) \leq e^{\lambda n(1 - p + \alpha)} \right\}\\
\leq& e^{-\lambda n(1 - p + \alpha)} \EE e^{\lambda \sum_{i=1}^n (1-X_i)}\\
=&    e^{-\lambda n(1 - p + \alpha)} \prod_{i=1}^n \EE e^{\lambda (1-X_i)}\\
\leq& \exp\left( -\lambda n (1 - p + \alpha)+ \sum_{i=1}^n\phi_{1-p_i}(\lambda) \right)
\end{align*}
with $\phi_{1-p_i}(\lambda) = \log(1+(1-p_i)+(1-p_i)e^\lambda)$. By concavity of the logarithm,
\begin{align*}
\sum_{i=1}^n &\frac{1}{n}\log(1+(1-p_i)+(1-p_i)e^\lambda)\\
&\leq \log (1 + (1-\frac{1}{n}\sum_{i=1}^n p_i) + (1 - \sum_{i=1}^n\frac{1}{n} p_i) e^{\lambda})\\
&=    \phi_{1-p}(\lambda) \: .
\end{align*}
The probability is then bounded by
\begin{align*}
\PP\left\{ S_n - np \leq -n\alpha \right\} \leq e^{-n[\lambda (1 - p + \alpha)- \phi_{1-p}(\lambda)]}
\end{align*}
Minimizing over $\lambda\geq 0$, we obtain
\begin{align*}
\PP\left\{ S_n - np \leq -n\alpha \right\} \leq e^{-n \phi_{1-p}^*(1-p+\alpha)}
\end{align*}
where $\phi_{1-p}^*$ is the convex conjugate of $\phi_{1-p}$.
\begin{align*}
\phi_{1-p}^*(x) &= x\log\frac{x}{1-p} + (1-x)\log\frac{1-x}{p} \: ,\\
\phi_{1-p}^*(1-p - \alpha) &= \phi_p^*(p+\alpha) \: .
\end{align*}

The second inequality follows from \citep{okamoto1959some}: $\phi_p^*(\alpha) \geq \frac{(p-\alpha)^2}{2p(1-p)}$ for $0\leq \alpha \leq p\leq 1/2$.
\end{proof}



\begin{lemma}\label{lemma:binary_concentration}
Let $X_s$ denote independent random variables satisfying $X_s \leq \EE(X_s)+ M_i$ for $1 \leq s \leq t$. We order the $X_s$ such that the $M_s$ are in increasing order. Let $X =\sum_{s=1}^t X_s$. Then for any $u\in[t]$ we have
\begin{align*}
&\PP\{ X\geq \EE X + \alpha \}\\
\leq&\exp\left(- \frac{\alpha^2}{2(Var(X) {+} \sum_{v=u}^t (M_v{-}M_u)^2 {+} M_u \frac{\alpha}{3})} \right) \: .
\end{align*}
\end{lemma}

\begin{lemma}\label{lemma:binary_concentration_t2}
Let $Y = C\log t -\sum_{s=1}^t Z_s$  where $Z_s\sim Ber(p)$ and $C$ is a constant. Then
\begin{align*}
\PP \{Y \geq C\log t - pt + \sqrt{5pt\log t}\} \leq \frac{1}{t^2}
\end{align*}
\end{lemma}
\begin{proof}
For any $\alpha>0$,
\begin{align*}
\PP\{Y \geq C\log t - pt + \alpha\}
&= \PP \{ \sum_{s=1}^t(-Z_s) \geq -pt + \alpha\}\\
&\leq \exp(-\frac{\alpha^2}{2pt(1 - p + \frac{\alpha}{3t})})
\end{align*}
from lemma~\ref{lemma:binary_concentration}, since $Z_s\leq \EE Z_s + p$. Let $\alpha = \sqrt{2apt\log t}$.
\begin{align*}
\PP\{X \geq \EE X + \alpha\}
&\leq \exp\left( -\frac{\alpha^2}{2(tp(1-p) + p \frac{\alpha}{3})} \right) \\
&= \exp\left( -\frac{a\log t}{(1-p +  \sqrt{\frac{2a p \log t}{9t}})} \right)\\
&\leq \exp\left( -\log t\frac{a}{1+a/18} \right)
\end{align*}
$\frac{a}{1+a/18} \geq 2 \Leftrightarrow a\geq \frac{9}{4}$. Taking $a=\frac{5}{2}$ leads to the claimed inequality. 
\end{proof}

\subsubsection{The Lambert Function}

\begin{definition}
The Lambert $W$ function is defined on $[-1/e,+\infty)$ by $W(x)e^{W(x)} = x$.
\end{definition}

The value $W(x)$ is close to $\log x - \log\log x$ for $x>e$:
\begin{align*}
\frac{\log \log x}{2 \log x} \leq W(x) {-} (\log x {-} \log\log x) \leq \frac{e}{e{-}1}\frac{\log \log x}{\log x}
\end{align*}
This function is increasing and such that for $x>0$ and $a>0$,
\begin{align*}
\frac{1}{x}\log(\frac{1}{x}) \geq a \Leftrightarrow x \leq \frac{W(a)}{a} \: .
\end{align*}

\subsection{Deterministic Passive Setting}

The rewards $X_t^{(i)}$ are 1-sub-Gaussian, meaning that for $\lambda>0$,
\begin{align*}
\log\EE e^{\lambda (X_t^{(i)} - \mu^{(i)})} \leq \frac{1}{2}\lambda^2 \: .
\end{align*}

Consider the following events, stating that all rewards are well concentrated around their means:
\begin{align*}
\mbox{for $i>1$, } \mathcal{E}_{i,t,s} &= \{ \frac{1}{s}\sum_{u=1}^s X_u^{(i)} - \mu^{(i)} \leq \sqrt{\frac{6\log t}{s}} \} \: ,\\
\mathcal{E}_{1,t,s} &= \{ \frac{1}{s}\sum_{u=1}^s X_u^{(1)} - \mu^{(1)} \geq \sqrt{\frac{6\log t}{s}} \} \: .
\end{align*}
As the rewards are 1-sub-Gaussian, for all $i\in[K]$,
\begin{align*}
\PP(\mathcal{E}_{i,t,s}) \geq 1 - \frac{1}{t^3} \: .
\end{align*}

Since the observations arrive deterministically, $F_i(t) = \lfloor\epsilon t p^{(i)}\rfloor$.
\begin{algorithm}
\caption{UCB with passive observations.}
\begin{algorithmic}[1]
  \State Pull each arm once.
  \Loop
  	  : at stage $t$,
  	  \State $i_t = \arg \max_{i} \overline{X}_t^{(i)} + \sqrt{\frac{6\log t}{O_i(t)}}$
      \State Pull arm $i_t$, observe $X_t^{(i)}$.
      \State Sample $f_t$.
      \State If $Z_t = 1$, observe $X_t^{(f_t)}$.
      \State Update $\overline{X}_t$, $N_i(t)$, $F_i(t)$, $O_i(t) = N_i(t)+F_i(t)$.
  \EndLoop
\end{algorithmic}
\end{algorithm}

We follow \citep{auer2002finite} to decompose the number of pulls of arm $i$,
\begin{align*}
\EE N_i(T) &\leq \EE \sum_{t=1}^T (\mathbb{I}\{O_i(t)\leq \frac{24\log t}{\Delta_i^2}\}\\
&\quad+ \mathbb{I}\{\overline{\mathcal{E}_{1,t,O_1(t)}}\} + \mathbb{I}\{\overline{\mathcal{E}_{i,t,O_i(t)}}\})\\
&\leq \EE \sum_{t=1}^T \mathbb{I}\{O_i(t)\leq \frac{24\log t}{\Delta_i^2}\}\\
&\quad+ \sum_{t=1}^T \sum_{s=1}^{2t} \PP\{\overline{\mathcal{E}_{1,t,s}}\} + \sum_{t=1}^T \sum_{s=1}^{2t} \PP\{\overline{\mathcal{E}_{i,t,s}}\}\:.
\end{align*}
We use here that $\PP(\mathcal{E}_{i,t,s}) \geq 1 - \frac{1}{t^3}$.
\begin{align*}
\EE N_i(T)&\leq \EE \sum_{t=1}^T \mathbb{I}\{N_i(t)\leq \frac{24\log t}{\Delta_i^2} - \epsilon p^{(i)} t\}\\
&\quad+ \sum_{t=1}^T \sum_{s=1}^{2t} \frac{1}{t^3} + \sum_{t=1}^T \sum_{s=1}^{2t} \frac{1}{t^3}\\
&= \EE \sum_{t=1}^T \mathbb{I}\{N_i(t)\leq \frac{24\log t}{\Delta_i^2} - \epsilon p^{(i)} t\}
+ \frac{2\pi^2}{3}\\
&\leq \sup_{t\in [1,T]}(\frac{24\log t}{\Delta_i^2} - \epsilon p^{(i)} t) + \frac{2\pi^2}{3}\\
&\leq \frac{24}{\Delta_i^2}\log(\frac{24}{\epsilon p^{(i)}\Delta_i^2 e}) + \frac{2\pi^2}{3} \: .
\end{align*}
The regret of UCB in this setting verifies
\begin{align*}
\EE R_T \leq \sum_{i=2}^K \frac{24}{\Delta_i} \log(\frac{24}{\epsilon p^{(i)}\Delta_i^2 e}) +  \frac{2\pi^2}{3}\sum_{i=2}^K \Delta_i \: .
\end{align*}

\subsection{Passive Setting}

As in the deterministic setting, we have the following inequality for $\EE N_i(T)$:
\begin{align*}
\EE N_i(T) \leq \EE\sum_{t=1}^T \mathbb{I}\{ N_i(t) + F_i(t) \leq \frac{24\log t}{\Delta_i^2} \} + \frac{2\pi^2}{3}
\end{align*}
The only difference here is that $F_i(t)$ is random. It is the sum of $t$ independent Bernoulli random variables $Ber(\epsilon_s p_s^{(i)})$. for $s\in[t]$. We consider here the static case, in which $\epsilon$ and all $p^{(i)}$ are independent of $t$.

Denote by $\mathcal{E}_{F,t}$ the event that
\begin{align*}
\frac{24\log t}{\Delta_i^2}-F_i(t)\leq \frac{24\log t}{\Delta_i^2} - \epsilon t p^{(i)} + \sqrt{5\epsilon p^{(i)}t\log t}
\end{align*}
$\PP(\mathcal{E}_{F,t}) \geq 1 - \frac{1}{t^2}$ from Lemma~\ref{lemma:binary_concentration_t2} .
\begin{align*}
\EE N_i(T)
&\leq \EE\sum_{t=1}^T \mathbb{I}\{ N_i(t) \leq \frac{24\log t}{\Delta_i^2} - F_i(t) \} + \frac{2\pi^2}{3}\\
&\leq \EE\sum_{t=1}^T \mathbb{I}(\{ N_i(t) \leq \frac{24\log t}{\Delta_i^2} - F_i(t)\}\cap \mathcal{E}_{F,t})\\
&\quad + \sum_{t=1}^T \PP( \overline{\mathcal{E}_{F,t}} ) + \frac{2\pi^2}{3}\\
&\leq \EE\sum_{t=1}^T \mathbb{I}\{ N_i(t) \leq \frac{24\log t}{\Delta_i^2} - \epsilon p^{(i)} t\\
&\qquad \qquad  \qquad   \qquad + \sqrt{5\epsilon p^{(i)} t\log t}\}\\
&\qquad + \sum_{t=1}^T\frac{1}{t^2} + \frac{2\pi^2}{3}
\end{align*}

\begin{align*}
\EE N_i(T)
&\leq \sup_{t\in[1,T]} \frac{24\log t}{\Delta_i^2} {-} t(\epsilon p^{(i)} {-} \sqrt{5\epsilon p^{(i)}\frac{\log t}{t}}) {+} \frac{5\pi^2}{6}
\end{align*}

We now need to determine the supremum of that function, which we denote by $g(t)$.

For  $\frac{t}{\log t} \geq \frac{20}{\epsilon p^{(i)}}$,
\begin{align*}
g(t)
&\leq \frac{24}{\Delta_i^2}\log t - \frac{1}{2}\epsilon p^{(i)} t\\
&\leq \frac{24}{\Delta_i^2}\log(\frac{48}{e \epsilon p^{(i)} \Delta_i^2})
\end{align*}

If Let $a = \frac{20}{\epsilon p^{(i)}}$. If $\frac{t}{\log t} \leq a$ then
\begin{align*}
-\frac{t}{a}e^{-t/a} \leq -\frac{1}{a} \Rightarrow t \leq -a W_{-1}(-\frac{1}{a})
\end{align*}
where $W_{-1}$ is the branch of the Lambert $W$ function defined on $[-1/e,0)$. It verifies for $u>0$ \citep{chatzigeorgiou2013bounds},
\begin{align*}
W_{-1}(-e^{-u-1}) \geq -1 -\sqrt{2u}-u \: .
\end{align*}
We obtain 
\begin{align*}
t \leq a (\log a + \sqrt{2\log a}) \leq \frac{5}{2}a\log a \: ,
\end{align*}
where the last inequality is valid for $a\geq e$. We can now bound $g(t)$:
\begin{align*}
g(t) &\leq \frac{24}{\Delta_i^2}\log t\\
&\leq \frac{24}{\Delta_i^2} \log(\frac{50}{\epsilon p^{(i)}}\log \frac{20}{\epsilon p^{(i)}}) \: .
\end{align*}

Overall,
\begin{align*}
g(t) \leq \frac{24}{\Delta_i^2}\left(\log\frac{50}{\epsilon p^{(i)}} {+} \max\left\{\log\frac{1}{e\Delta_i^2} ,\log\log\frac{20}{\epsilon p^{(i)}} \right\}\right) . 
\end{align*}

\subsection{Active Setting}

\subsubsection{OCUCBn + ETC}

At epoch $m$, OCUCBn is used to decide which arm to pull, starting from zero. During the same epoch, ETC is used with the free observations (but with confidence levels adapted to epoch $m+1$). Epoch $m$ has length $d_m$.

At epoch $m$, ETC has eliminated arms such that only a set $S_m\subseteq [K]$ remains. The regret of OCUCBn on this epoch is given by the following lemma, taken from \citep{lattimore2016regret}.
\begin{lemma}\label{lemma:ocucb_regret}
If $\rho \in [1/2, 1]$ and $\eta > 1$, then the expected regret on epoch $m$ is
\begin{align*}
&\EE R_{(m)}(S_m)\\
\leq&  \sum_{i\in S_m:\Delta_i>0} \frac{C_\eta}{\Delta_i} \log(\max\{\frac{d_m\Delta_i^2}{k_{i,\rho(S_m)}},1\}\log(d_m))\\
&\quad + \sum_{i\in S_m:\Delta_i>0}C_\eta \Delta_i \: ,
\end{align*}
where $k_{i,\rho}(S_m) = \sum_{j\in S_m} \min\{1, \frac{\Delta_i^{2\rho}}{\Delta_j^{2\rho}}\}$ and $C_\eta > 0$ is a constant that depends only on $\eta$. Furthermore, for all $\rho \in [0, 1]$ it holds that $\limsup_{d_m\to+\infty}\frac{\EE R_{(m)}(S_m)}{\log d_m} \leq \sum_{i\in S_m:\Delta_i>0} \frac{2\eta}{\Delta_i}$ .
\end{lemma}

We denote the upper bound of the regret by $B_m(S_m)$. The mapping $A\mapsto B_m(\{1\}\cup A)$ is increasing with respect to set inclusion, i.e. if $A\subset B \subset [K]$, $B_m(\{1\}\cup A) \leq B_m(\{1\}\cup B)$.
The total regret of our algorithm will be the sum over epochs of these terms. The only unknown at this point is the set $S_m$, which is determined by the Explore-Then-Commit algorithm on the free observations.

There exists $i_m\in[K+1]$ to be determined later such that with probability $1-\delta_m$, $\delta_m$ also to be computed later, ETC has eliminated all arms with $i\geq i_m$ before the start of epoch $m$, and has not eliminated arm 1. In this case the regret is upper bounded by $B_m ([i_m-1])$. Otherwise the regret is bounded by $\delta_m d_m$ . Taking $\delta_m$ of order $\frac{1}{m^2d_m}$ or smaller leads to a finite expected regret for this bad case.

Let $H_{i,\rho}$ and $H_{i,\rho}^{(m)}$ be defined by
\begin{align*}
H_{i,\rho} &= \frac{i}{\Delta_i^2} + \sum_{j=i+1}^K \frac{1}{\Delta_j^{2\rho}\Delta_i^{2(1-\rho)}} \: ,\\
H_{i,\rho}^{(m)} &= \frac{i}{\Delta_i^2} + \sum_{j=i+1}^{i_m}\frac{1}{\Delta_j^{2\rho}\Delta_i^{2(1-\rho)}} \: .
\end{align*}
For all $m$, $H_{i,\rho}^{(m)}\leq H_{i,\rho}$.

\paragraph{Regret under concentration.}
\begin{lemma}
Let $m_i = \min(\lceil\log_2\log_2 T\rceil, \min\{m\in\N \: : \: i\geq i_m\})$. If $S_m \subset [i_m-1]$ and $1\in S_m$ for all $m\in \N$, then the regret of the algorithm is
\begin{align*}
\EE R_T
\leq & \sum_{i=2}^K \frac{C_\eta}{\Delta_i}4 \log \max(\frac{d_{m_i-1}}{H_{i,\rho}^{(m_i)}},\frac{H_{i,\rho}}{\sqrt{H_{i,\rho}^{(m_i)}}})\\
&\qquad + \sum_{i=2}^K \frac{C_\eta}{\Delta_i} (m_i+m_i^2) + C_\eta \sum_{i=2}^K \Delta_i  \: .
\end{align*}
\end{lemma}

\begin{proof}
From Lemma~\ref{lemma:ocucb_regret}, the sum of the $i$ terms over all epochs is
\begin{align*}
B_i(T)
&\leq C_\eta\sum_{m\leq m_i} \Delta_i + \frac{1}{\Delta_i}\log\max\{1,\frac{d_m}{H_{i,\rho}^{(m)}}\}\\
&\qquad \qquad + \frac{1}{\Delta_i}\log\log(d_m) \: .
\end{align*}
Let $m_{i,0} = \min\{m\in\N \: : \: d_m \geq H_{i,\rho}^{(m)}\}$.
\begin{align*}
\sum_{m=0}^{m_i}\log(\max\{1,\frac{d_m}{H_{i,\rho}^{(m)}}\})
&= \sum_{m=m_{i,0}}^{m_i}\log\frac{d_m}{H_{i,\rho}^{(m)}}\\
&\leq \sum_{m=m_{i,0}}^{m_i}\log\frac{d_m}{H_{i,\rho}^{(m_i)}}\: ,
\end{align*}
For $a,b\geq 0$, such that $d_a \geq H_{i,\rho}^{(a)}$,
\begin{align*}
\sum_{m=a}^b \log d_m &= \log(2)(2^{b+1} - 2^a)\\
&= \log d_{b+1} - \log d_a\\
&\leq \log d_{b+1} - \log H_{i,\rho}^{(a)}\\
&\leq \log d_{b+1} - \log H_{i,\rho}^{(b)} \: ,
\end{align*}
For $H_{i,\rho}^{(m_i)} \geq 1$, if  $m_i\geq m_{i,0}+2$,
\begin{align*}
\sum_{m=0}^{m_i}\log(\max\{1,\frac{d_m}{H_{i,\rho}^{(m)}}\})
&\leq \log \frac{d_{m_i+1}}{(H_{i,\rho}^{(m_i)})^{m_i - m_{i,0}+2}}\\
&\leq \log \frac{d_{m_i+1}}{(H_{i,\rho}^{(m_i)})^4}\\
&=    4\log \frac{d_{m_i-1}}{H_{i,\rho}^{(m_i)}} \: .
\end{align*}
If $m_i=m_{i,0}+1$, we have
\begin{align*}
d_{m_i} \leq d_{m_{i,0}-1}^4 \leq (H_{i,\rho}^{(m_{i,0}-1)})^4
\end{align*}
and the sum take the form
\begin{align*}
\sum_{m=0}^{m_i}\log(\max\{1,\frac{d_m}{H_{i,\rho}^{(m)}}\}) \leq 2\log(\frac{(H_{i,\rho})^2}{H_{i,\rho}^{(m_i)}}) \: .
\end{align*}
If $m_i = m_{i,0}$ then the sum is reduced to one term and $d_{m_i} \leq (H_{i,\rho}^{(m_{i,0}-1)})^2$. We obtain
\begin{align*}
\sum_{m=0}^{m_i}\log(\max\{1,\frac{d_m}{H_{i,\rho}^{(m)}}\}) \leq \log(\frac{(H_{i,\rho})^2}{H_{i,\rho}^{(m_i)}}) \: .
\end{align*}

And the $\log\log$ term is
\begin{align*}
\sum_{m=0}^{m_i} \log \log d_m \leq m_i(m_i+1) \frac{\log 2}{2} \: .
\end{align*}
\end{proof}

\paragraph{Explore-Then-Commit.}

We run the Explore-Then-Commit algorithm presented as Algorithm~\ref{algo:etc}.

The aim of the ETC subroutine during epoch $m$ is to gather information on each arm to eliminate bad arms from epoch $m+1$. We prove that it achieves this goal by proving two facts:
\begin{enumerate}
\item the total number of free observations available during epoch $m$ is big enough with respect to its expectation $\epsilon d_m$,
\item if the total number of free observations is greater than a threshold $\tau_i$, then all arms $j\geq i$ are eliminated.
\end{enumerate}

\begin{lemma}\label{lemma:etc_enough_observations}
Let $\epsilon\in(0,\frac{1}{2}]$ be the probability of getting a free information and suppose that $\frac{d_m}{\log (d_{m+1}\log d_{m+1})}\geq \frac{8}{\epsilon}$. With probability greater than $1 - \frac{1}{d_{m+1}\log d_{m+1}}$,
\begin{align*}
\sum_{t\in (m)}Z_t \geq \frac{1}{2}\epsilon d_m \: .
\end{align*}
\end{lemma}

\begin{proof}
The constraint $\frac{d_m}{\log (d_{m+1}\log d_{m+1})}\geq \frac{8}{\epsilon}$ is equivalent to $\frac{1}{2}\epsilon d_m \leq \epsilon d_m - \sqrt{2\epsilon d_m \log (d_{m+1}\log d_{m+1})}$. 
From Lemma~\ref{lemma:binary_concentration_variance},
\begin{align*}
&\PP\left\{ \sum_{t\in (m)}Z_t \leq \epsilon d_m - \sqrt{2\epsilon d_m \log (d_{m+1}\log d_{m+1})} \right\}\\
&\leq  \exp\left( - \frac{1}{2d_m\epsilon(1-\epsilon)} (\sqrt{2\epsilon d_m \log(d_m)})^2 \right)\\
&=  \frac{1}{d_{m+1}\log (d_{m+1})} \: .
\end{align*}
\end{proof}

\begin{lemma}
Let the total number of free observations during epoch $m\in\N$ be $\tau \leq d_m$. Then with probability greater than $1 - 12K\frac{d_m}{T}\sqrt{\log(\frac{T}{d_m})}$ the ETC algorithm with $\alpha = 1$ and horizon $T$ ensures that $1\in S_{m+1}$.

Furthermore, if
\begin{align*}
\tau \geq K + C_\alpha (i\frac{\log(T\Delta_i^2)}{\Delta_i^2} + \sum_{j=i+1}^K\frac{\log(T\Delta_j^2)}{\Delta_j^2} ) \: ,
\end{align*}
with $C_\alpha = 8(\alpha+1)$, then with probability greater than $1 - 12K\frac{d_m}{T}\sqrt{\log(\frac{T}{d_m})}$, the ETC algorithm with parameter $\alpha$ and horizon $T$ ensures that
\begin{align*}
1 \in S_{m+1} \: ,\quad
j\geq i \Rightarrow j\notin S_{m+1} \: .
\end{align*}
\end{lemma}

\begin{proof}
The proof of this statement is the object of section~\ref{section:ETC}.
\end{proof}

\begin{lemma}\label{lemma:etc_total_obs_to_S_m}
Let $\epsilon\in(0,\frac{1}{2}]$ be the probability of getting a free information. Let the epoch number verify $2^m-m\geq \log_2(\frac{1}{\epsilon}H_{i,1}10C_\alpha\log 2)$ for some $i\in[K]$. Let $\delta_{m+1}$ be defined by 
\begin{align*}
\delta_{m+1} = \frac{1}{d_{m+1}}\frac{1+12K\sqrt{\log(d_{m+1}\log d_{m+1})}}{\log d_{m+1}} \:,
\end{align*}
With probability greater than $1 - \delta_{m+1}$, the ETC algorithm with $\alpha = 1$ and $T = d_{m+1}^{3/2} \log(d_{m+1})$ ensures that
\begin{align*}
1 \in S_{m+1} \: ,\quad
j\geq i \Rightarrow j\notin S_{m+1} \: .
\end{align*}
\end{lemma}

\begin{proof}
The constraint on $m$ ensures that
\begin{align*}
2^m-m &\geq \log_2(\frac{1}{\epsilon}H_{i,1}10C_\alpha\log 2)\\
\Rightarrow \frac{2^{2^m}}{2^m\log 2} &\geq \frac{1}{\epsilon}H_{i,1}10C_\alpha\\
\Leftrightarrow \frac{d_m}{\log d_m} &\geq \frac{1}{\epsilon}H_{i,1}10C_\alpha\\
\Leftrightarrow \frac{1}{2}\epsilon d_m &\geq  C_\alpha H_{i,1} \log(d_{m+1}^{5/2})\\
\Rightarrow \frac{1}{2}\epsilon d_m &\geq C_\alpha H_{i,1} \log(e d_{m+1}^{3/2}\log d_{m+1})\\
\Rightarrow \frac{1}{2}\epsilon d_m &\geq K+C_\alpha H_{i,1} \log(d_{m+1}^{3/2}\log d_{m+1}\Delta_i^2) \: .
\end{align*}
Now remark that this proves also that
\begin{align*}
\frac{d_m}{\log(d_{m+1}\log d_{m+1})} \geq \frac{8}{\epsilon} \: ,
\end{align*}
hence Lemma~\ref{lemma:etc_enough_observations} is fully applicable. With probability $1-\frac{1}{d_{m+1}\log d_{m+1}}$, the total number of free observations during epoch $m$ is greater than $\frac{1}{2}\epsilon d_m$. When this happens, the hypotheses of Lemma~\ref{lemma:etc_total_obs_to_S_m} are verified and with probability greater than $1 - \frac{1}{d_{m+1}}\frac{12K\sqrt{\log(d_{m+1}\log d_{m+1})}}{\log d_{m+1}}$, 
\begin{align*}
1 \in S_{m+1} \: ,\quad
j\geq i \Rightarrow j\notin S_{m+1} \: .
\end{align*}
\end{proof}

For any epoch $m$, we can now define $i_m \in [K]$ such that for all $j\geq i_m$ and $m'\geq m+1$, $j\notin S_{m'}$ .
\begin{align*}
i_m = \min&\{ i\in[K+1] \: : \\
& 2^m-m \geq \log_2(\frac{1}{\epsilon}H_{i,1}10C_\alpha\log 2) \}
\end{align*}
with the convention that the minimum has value $K+1$ if the set is empty.

\paragraph{Putting all together.}
Let $\mathcal{C}$ be the event that $1\in S_m$ and $S_{m+1} \subset [i_m-1]$ for all $m$. Then
\begin{align*}
\PP\{\overline{\mathcal{C}}\}\leq&\sum_{m=0}^{\lceil\log_2\log_2 T\rceil} d_{m+1} \delta_{m+1}\\
= &\sum_{m=0}^{\lceil\log_2\log_2 T\rceil} \frac{1+12K\sqrt{\log(d_{m+1}\log d_{m+1})}}{\log d_{m+1}}\\
\leq & 51 K \: .
\end{align*}
The total regret is
\begin{align*}
\EE R_T & \leq
C_\eta \sum_{i=2}^K \frac{1}{\Delta_i} 4\log \max(\frac{d_{m_i-1}}{H_{i,\rho}^{(m_i)}},\frac{H_{i,\rho}}{\sqrt{H_{i,\rho}^{(m_i)}}})\\
&\quad + C_\eta \sum_{i=2}^K \frac{1}{\Delta_i} (m_i+m_i^2) + C_\eta \sum_{i=2}^K \Delta_i + 51 K  \: .
\end{align*}

\begin{lemma}\label{lemma:epoch_m_i_upper_bound}
The epoch length $d_{m_i-1}$ verify
\begin{align*}
d_{m_i-1} \leq \frac{25C_\alpha}{\log 2}\frac{H_{i,1}}{\epsilon}\log(\frac{10 C_\alpha H_{i,1}}{\epsilon \log 2})
\end{align*}
\end{lemma}
\begin{proof}
By definition of $m_i$ as the first integer such that $2^m-m$ is greater than a value, $2^{m_i-1}-(m_i-1)$ is smaller than the same value.
\begin{align*}
2^{m_i-1} - (m_i-1) \leq \log_2(\frac{1}{\epsilon}H_{i,1}10C_\alpha\log 2) \: .
\end{align*}
Let $C = 10C_\alpha$. This last inequality is equivalent to
\begin{align*}
&\frac{d_{m_i-1}}{\log d_{m_i-1}}\leq \frac{C}{\epsilon} H_{i,1}\\
\Leftrightarrow &-d_{m_i-1}\frac{\epsilon \log 2}{C H_{i,1}} \exp(-d_{m_i-1}\frac{\epsilon \log 2}{C H_{i,1}}) \geq -\frac{\epsilon \log 2}{C H_{i,1}}\\
\Leftrightarrow & d_{m_i-1}\frac{\epsilon \log 2}{C H_{i,1}} \leq -W_{-1} (-\frac{\epsilon \log 2}{C H_{i,1}}) \: ,
\end{align*}
where $W_{-1}$ is the branch of the Lambert $W$ function defined on $[-1/e,0)$. It verifies for $u>0$ \citep{chatzigeorgiou2013bounds},
\begin{align*}
W_{-1}(-e^{-u-1}) \geq -1 -\sqrt{2u}-u \: .
\end{align*}
We obtain 
\begin{align*}
d_{m_i-1}&\leq \frac{C H_{i,1}}{\epsilon \log 2}\left(\sqrt{2\log(\frac{C H_{i,1}}{\epsilon \log 2})} + \log(\frac{C H_{i,1}}{\epsilon \log 2}) \right)\\
&\leq \frac{5}{2}\frac{C H_{i,1}}{\epsilon \log 2}\log(\frac{C H_{i,1}}{\epsilon \log 2})
\end{align*}
\end{proof}

The leading term of the regret has the form
\begin{align*}
&4C_\eta\sum_{i=2}^K \frac{1}{\Delta_i} \log(\frac{1}{\epsilon}\frac{H_{i,1}}{H_{i,\rho}^{(m_i)}})
\end{align*}

\begin{lemma}\label{lemma:ratio_H}
The quantities $H_{i,1}$ and $H_{i,\rho}^{(m_i)}$ are such that for all $\rho\in[1/2,1]$,
\begin{align*}
\sum_{i=2}^K \frac{1}{\Delta_i}\log\left(\frac{H_{i,1}}{H_{i,\rho}^{(m_i)}} \right) \leq\sum_{i=2}^K\frac{1}{\Delta_i} \log(1+\log K) \: .
\end{align*}
\end{lemma}

\begin{proof}
We start by simplifying the expressions of the ratio of the $H$ constants,
\begin{align*}
H_{i,\rho}^{(m_i)} &= \frac{i}{\Delta_i^2} + \sum_{j=i+1}^{i_{m_i}} \frac{1}{\Delta_j^{2\rho}\Delta_i^{2(1-\rho)}} \geq \frac{i}{\Delta_i^2}\: ,\\
\frac{H_{i,1}}{H_{i,1}^{(m_i)}} &\leq 1 + \sum_{j=i+1}^K \frac{\Delta_i^2}{i\Delta_j^2}\: .
\end{align*}
By concavity of the logarithm,
\begin{align*}
&\sum_{i=2}^K \frac{1}{\Delta_i}\log(1 + \sum_{j=i+1}^K \frac{\Delta_i^2}{i\Delta_j^2})\\
\leq & (\sum_{i=2}^K \frac{1}{\Delta_i}) \log(1 + (\sum_{i=2}^K \frac{1}{\Delta_i})^{-1}\sum_{i=2}^K \sum_{j=i+1}^K \frac{\Delta_i}{i\Delta_j^2}))
\end{align*}
Let $A = \sum_{i=2}^K \sum_{j=i+1}^K \frac{\Delta_i}{i\Delta_j^2}$.
\begin{align*}
A = \sum_{j=3}^K\sum_{i=2}^{j-1} \frac{\Delta_i}{i\Delta_j^2}
&\leq \sum_{j=3}^K\sum_{i=2}^{j-1} \frac{1}{i\Delta_j}\\
&\leq \sum_{j=2}^K\frac{1}{\Delta_j}\log(j)  \: .
\end{align*}
$(\sum_{i=2}^K \frac{1}{\Delta_i})^{-1} A = \sum_{j=2}^K \log(j) \lambda_j$ with $\lambda_j = \frac{1/\Delta_j}{\sum_{k=2}^K1/\Delta_k}$. The values $(\lambda_j)_{2\leq j \leq K}$ are such that $\sum_{j=2}^K \lambda_j = 1$ and $1>\lambda_2 \geq \lambda_3 \geq \ldots \geq \lambda_K>0$.
\begin{align*}
\sum_{j=2}^K \log(j) \lambda_j \leq \max_{j=2,\ldots,K} \log(j) = \log K\: .
\end{align*}
\end{proof}

We can now give the final form of the regret bound.

\begin{lemma}
The regret of the active algorithm is
\begin{align*}
\EE R_T & \leq 
\sum_{i=2}^{K} \frac{4 C_\eta}{\Delta_i} \max\left\{\log(\frac{1}{\epsilon}),\log \sqrt{H_{i,\rho}}\right\}\\
&\qquad + 51K +O(\sum_{i=2}^K \frac{1}{\Delta_i}(\log\log \frac{H_{i,1}}{\epsilon})^2) \: .
\end{align*}
\end{lemma}

\begin{proof}
\begin{align*}
\EE R_T & \leq
C_\eta \sum_{i=2}^K \frac{1}{\Delta_i}  4\log\max\left(\frac{d_{m_i-1}}{H_{i,\rho}^{(m_i)}},\frac{H_{i,\rho}}{\sqrt{H_{i,\rho}^{(m_i)}}}\right)\\
& + C_\eta \sum_{i=2}^K \frac{1}{\Delta_i}(m_i  + m_i^2) + C_\eta \sum_{i=2}^K \Delta_i + 51K \: .
\end{align*}

The inequality on $d_{m_i-1}$ of lemma~\ref{lemma:epoch_m_i_upper_bound} implies that
\begin{align*}
m_i \leq 1 + \log_2\log_2\left[\frac{25C_\alpha}{\log 2}\frac{H_{i,1}}{\epsilon}\log(\frac{10 C_\alpha H_{i,1}}{\epsilon \log 2})\right] \: .
\end{align*}
And by definition, $m_i \leq 1+\log_2\log_2 T$. The first term of the max is
\begin{align*}
\log \frac{d_{m_i-1}}{H_{i,\rho}^{(m_i)}}
&\leq \log \left(\frac{25C_\alpha}{\log(2) \epsilon}\log\frac{10C_\alpha H_{i,1}}{\epsilon \log 2}\right)\\
&\qquad + \log \frac{H_{i,1}}{H_{i,\rho}^{(m_i)}} \: .
\end{align*}
The second term is
\begin{align*}
\log \frac{H_{i,\rho}}{\sqrt{H_{i,\rho}^{(m_i)}}} &= \log \frac{H_{i,\rho}}{H_{i,\rho}^{(m_i)}} + \log \sqrt{H_{i,\rho}^{(m_i)}}\: .
\end{align*}
The maximum is then
\begin{align*}
&\sum_{i=2}^K \frac{1}{\Delta_i} \log\max\left(\log\frac{d_{m_i-1}}{H_{i,\rho}^{(m_i)}},\frac{H_{i,\rho}}{\sqrt{H_{i,\rho}^{(m_i)}}}\right)\\
\leq & \sum_{i=2}^K \frac{1}{\Delta_i} \log \frac{H_{i,\rho}}{H_{i,\rho}^{(m_i)}}\\
&\quad + \sum_{i=2}^K \frac{1}{\Delta_i} \max\Big\{ \log\left( \frac{25C_\alpha}{\log(2) \epsilon}\log\frac{10C_\alpha H_{i,1}}{\epsilon \log 2} \right),\\
&\qquad\qquad \qquad\log \sqrt{H_{i,\rho}^{(m_i)}} \Big\}\\
\leq& \sum_{i=2}^K \frac{1}{\Delta_i} \log \frac{H_{i,\rho}}{H_{i,\rho}^{(m_i)}}\\
&\quad + \sum_{i=2}^K \frac{1}{\Delta_i} \max\Big\{ \log\left( \frac{1}{\epsilon} \right),\log \sqrt{H_{i,\rho}^{(m_i)}} \Big\}\\
&\qquad + O(\sum_{i=2}^K \frac{1}{\Delta_i}\log\log \frac{H_{i,1}}{\epsilon}) \: .
\end{align*}
where the $O(\sum_{i=2}^K \frac{1}{\Delta_i}\log\log \frac{H_{i,1}}{\epsilon})$ term regroups the constant and doubly logarithmic terms in the previous expression.
The sum over $i$ of the  $\log \frac{H_{i,1}}{H_{i,\rho}^{(m_i)}}$ terms is bounded in Lemma~\ref{lemma:ratio_H} by $\sum_{i=2}^K\frac{1}{\Delta_i}(1+\log K)$.
\begin{align*}
\sum_{i=2}^K &\frac{1}{\Delta_i} \log\max\left(\log\frac{d_{m_i-1}}{H_{i,\rho}^{(m_i)}},\frac{H_{i,\rho}}{\sqrt{H_{i,\rho}^{(m_i)}}}\right)\\
\leq & \sum_{i=2}^K \frac{1}{\Delta_i} \max\Big\{ \log\left( \frac{1}{\epsilon} \right),\log \sqrt{H_{i,\rho}^{(m_i)}} \Big\}\\
&\qquad + O(\sum_{i=2}^K \frac{1}{\Delta_i}\log\log \frac{H_{i,1}}{\epsilon}) \: .
\end{align*}

\end{proof}

\subsubsection{The Explore-Then-Commit Algorithm.}\label{section:ETC}

The ETC algorithm with parameter $\alpha>0$ discards arm $i\in[K]$ at a comparison stage $t\in\N$ if for some other arm $j\in[K]$ not yet eliminated,
\begin{align*}
\hat{\mu}_t^{(i)} + \sqrt{\frac{2\alpha}{s}\log(\frac{T}{s})} < \hat{\mu}_t^{(j)} - \sqrt{\frac{2\alpha}{s}\log(\frac{T}{s})}
\end{align*}
where $s\in\N^*$ is their common number of observations at stage $t$.

\begin{lemma}
Define the event that all arms have empirical means concentrated around their expectations up to stage $\tau$, with parameter $c>0$, by
\begin{align*}
\mathcal{C}_{\tau,c} = \left\{ \forall i\in[K],\forall s\leq \tau, |\hat{\mu}_s^{(i)} {-} \mu^{(i)}| {\leq} \sqrt{\frac{2c}{s}\log(\frac{T}{s})} \right\}
\end{align*}
For $c=1$ and $\tau \leq 0.2 T$, this event happens with probability
\begin{align*}
\PP(\mathcal{C}_{\tau,1}) \geq 1 - 12K\frac{\tau}{T}\sqrt{\log(\frac{T}{\tau})} \: .
\end{align*}
\end{lemma}
\begin{proof}
The result for one arm is Lemma~\ref{lemma:martingale_concentration}. An union bound over the arms gives the wanted inequality.
\end{proof}

\begin{lemma}
Suppose that $\alpha>c$. If the concentration event $\mathcal{C}_{\tau,c}$ holds and arm $i\in[K]$ is eliminated by arm $j\in[K]$ at a stage $t\leq \tau$, then $\mu_j>\mu_i$. In particular, arm 1 is not eliminated before $\tau$.
\end{lemma}
\begin{proof}
Arm $i$ is eliminated by arm $j$ if
\begin{align*}
\hat{\mu}_i + \sqrt{\frac{2\alpha\log(T/s)}{s}} < \hat{\mu}_j - \sqrt{\frac{2\alpha\log(T/s)}{s}} \: .
\end{align*}
Then one of three inequalities is true for $a\in(0,1)$ such that $a^2\alpha\geq c$,
\begin{align*}
&\hat{\mu}_i \leq \mu_i - a\sqrt{\frac{2\alpha\log(T/s)}{s}}\\
\mbox{or } & \hat{\mu}_j \leq \mu_j  + a\sqrt{\frac{2\alpha\log(T/s)}{s}} \\
\mbox{or } & 2\sqrt{\frac{8\log(T/s)}{s}} < \frac{\mu_j - \mu_i}{1-a} \: .
\end{align*}
Since concentration holds and $a^2\alpha \geq c$, the two first inequalities are false, so that the third one is true. Hence $\mu_j - \mu_i$ is positive.
\end{proof}

\begin{lemma}
Let $\mathcal{A}$ be the ETC algorithm with parameter $\alpha>c$ and comparisons done each time all remaining arms are incremented by 1.
Let $C_\alpha = 8c(\frac{\alpha}{c}+1)$. If the concentration event $\mathcal{C}_{\tau,c}$ holds and the number of observations of all arms is smaller than $\tau$, the time at which arm $i$ is discarded by $\mathcal{A}$ is
\begin{align*}
\tau_i \leq K + C_\alpha\left[ \frac{i}{\Delta_i^2}W(\frac{T\Delta_i^2}{C_\alpha}) + \sum_{j=i+1}^K\frac{1}{\Delta_j^2}W(\frac{T\Delta_j^2}{C_\alpha}) \right] \: ,
\end{align*}
where $W$ is the Lambert $W$ function. We denote this bound by $H_i(T)$. After $H_i(T)$, all arms $j \in \{i,\ldots,K\}$ are eliminated and arm 1 is not eliminated.
\end{lemma}

\begin{proof}
Let $a>0$ be such that $a^2\alpha\geq c$. If an arm $i$ is not eliminated by arm 1 when they were both observed $s$ times then
\begin{align*}
&\hat{\mu}_i \geq \mu_i + a\sqrt{\frac{2\alpha\log(T/s)}{s}}\\
\mbox{or } & \hat{\mu}_1 \leq \mu_1  - a\sqrt{\frac{2\alpha\log(T/s)}{s}} \\
\mbox{or } & 2\sqrt{\frac{2\alpha\log(T/s)}{s}} > \frac{\Delta_i}{1+a} \: .
\end{align*}
The two first inequalities are false from concentration. The third inequality leads to
\begin{align*}
\frac{1}{s}\log(\frac{T}{s}) &> \frac{\Delta_i^2}{8\alpha(1+a)^2}\\
\Rightarrow s &\leq  \frac{8\alpha(1+a)^2}{\Delta_i^2} W(\frac{T\Delta_i^2}{8\alpha(1+a)^2}) \: .
\end{align*}
If an arm $i$ is eliminated by another arm, by the same reasoning,
\begin{align*}
s \geq \frac{8\alpha(1-a)^2}{\Delta_i^2} W(\frac{T\Delta_i^2}{8\alpha(1-a)^2}) \: .
\end{align*}
The comparison is done at each time for which all remaining arms have been incremented by 1, such that if an arm $i\in[K]$ is eliminated at a number of observations $s_i$, it was not eliminated at $s_i-1$. At the elimination stage of arm $i$, its number of observations $s_i$ verify
\begin{align*}
s_i &\geq \frac{8\alpha(1-a)^2}{\Delta_i^2} W(\frac{T\Delta_i^2}{8\alpha(1-a)^2})\\
s_i &\leq 1 + \frac{8\alpha(1+a)^2}{\Delta_i^2} W(\frac{T\Delta_i^2}{8\alpha(1+a)^2})
\end{align*}

The concentration holds if $a^2\geq \frac{c}{\alpha}$. Taking the smallest valid parameter $a^2=\frac{c}{\alpha}$, the number of observations of arm $i$ at elimination is
\begin{align*}
s_i &\geq (\frac{\alpha}{c}-1)\frac{8c}{\Delta_i^2} W(\frac{T\Delta_i^2}{8c(\alpha/c-1)})\\
s_i &\leq 1 + (\frac{\alpha}{c}+1)\frac{8c}{\Delta_i^2} W(\frac{T\Delta_i^2}{8c(\alpha/c+1)})
\end{align*}

Now if the total number of observations is greater than $\tau_i$ then $\tau_i/K \geq 1 + \frac{C_\alpha}{\Delta_K^2}W(\frac{T\Delta_K^2}{C_\alpha})$, and arm $K$ is eliminated. The remaining number of observations of arms 1 to $K-1$ is greater than
\begin{align*}
K-1 + C_\alpha\left[ \frac{i}{\Delta_i^2}W(\frac{T\Delta_i^2}{C_\alpha}) + \sum_{j=i+1}^{K-1}\frac{1}{\Delta_j^2}W(\frac{T\Delta_j^2}{C_\alpha}) \right] \: .
\end{align*}
A repetition of the same line of reasoning gives that all arms $j\geq i$ are eliminated.

\end{proof}

\end{document}